\documentclass{article}

    \usepackage[preprint,nonatbib]{neurips_2024}

\usepackage[utf8]{inputenc} 
\usepackage[T1]{fontenc}    
\usepackage{hyperref}       
\usepackage{url}            
\usepackage{booktabs}       
\usepackage{amsfonts}       
\usepackage{nicefrac}       
\usepackage{microtype}      
\usepackage{xcolor}         
\usepackage{wrapfig}

\usepackage{microtype}
\usepackage{graphicx}
\usepackage{booktabs} 
\usepackage{float}
\usepackage{hyperref}
\usepackage{url}
\usepackage[numbers]{natbib}

\newcommand{\smib}[0]{\textsc{Mime}}
\newcommand{\an}[0]{\textsc{An}}
\newcommand{\anls}[0]{\textsc{An-ls}}
\newcommand{\wan}[0]{\textsc{Wan}}
\newcommand{\epr}[0]{\textsc{Epr}}
\newcommand{\role}[0]{\textsc{Role}}
\newcommand{\smilee}[0]{\textsc{Smile}}
\newcommand{\emm}[0]{\textsc{Em}}
\newcommand{\emapl}[0]{\textsc{Em-Apl}}
\newcommand{\lagc}[0]{\textsc{Plc}}
\newcommand{\proposed}[0]{\textsc{Crisp}}
\newcommand{\lln}[0]{\textsc{Ll}}
\newcommand{\llr}[0]{\textsc{Ll-R}}
\newcommand{\llct}[0]{\textsc{Ll-Ct}}
\newcommand{\llcp}[0]{\textsc{Ll-Cp}}
\newcommand{\boost}[0]{\textsc{BoostLU}}

\usepackage{amsmath}
\usepackage{amssymb}
\usepackage{mathtools}
\usepackage{amsthm}
\usepackage{bm}
\usepackage{algorithm}
\usepackage{algorithmic}
\usepackage{multirow}
\usepackage{graphicx}
\usepackage{caption}
\usepackage{subcaption}
\usepackage{array}
\usepackage[shortlabels,inline]{enumitem}

\usepackage[capitalize,noabbrev]{cleveref}

\theoremstyle{plain}
\newtheorem{theorem}{Theorem}[section]

\newtheorem{lemma}[theorem]{Lemma}

\theoremstyle{definition}

\theoremstyle{remark}

\title{Can Class-Priors Help Single-Positive \\ Multi-Label Learning?}

\author{%
  Biao Liu \\
  School of Computer Science and Engineering \\
  Southeast University\\
  Nanjing, China \\
  \texttt{liubiao01@seu.edu.cn} \\
  \And
  Ning Xu \\
  School of Computer Science and Engineering \\
  Southeast University\\
  Nanjing, China \\
  \texttt{xning@seu.edu.cn} \\
  \And
  Jie Wang \\
  School of Computer Science and Engineering \\
  Southeast University\\
  Nanjing, China \\
  \texttt{wangjie2022@seu.edu.cn} \\
  \And
  Xin Geng \\
  School of Computer Science and Engineering \\
  Southeast University\\
  Nanjing, China \\
  \texttt{xgeng@seu.edu.cn} \\
}

\begin{document}

\maketitle

\begin{abstract}
Single-positive multi-label learning (SPMLL) is a weakly supervised multi-label learning problem, where each training example is annotated with only one positive label. Existing SPMLL methods typically assign pseudo-labels to unannotated labels with the assumption that prior probabilities of all classes are identical.
However, the class-prior of each category may differ significantly in real-world scenarios, which makes the predictive model not perform as well as expected due to the unrealistic assumption on real-world application.
To alleviate this issue, a novel framework named {\proposed}, i.e., Class-pRiors Induced Single-Positive multi-label learning, is proposed. Specifically, a class-priors estimator is introduced, which can estimate the class-priors that are theoretically guaranteed to converge to the ground-truth class-priors. In addition, based on the estimated class-priors, an unbiased risk estimator for classification is derived, and the corresponding risk minimizer can be guaranteed to approximately converge to the optimal risk minimizer on fully supervised data.
Experimental results on ten MLL benchmark datasets demonstrate the effectiveness and superiority of our method over existing SPMLL approaches.    
\end{abstract}

\section{Introduction}\label{sec:introduction}
  
Multi-label learning (MLL) is a learning paradigm that aims to train a model on examples associated with multiple labels to accurately predict relevant labels for unknown instances \cite{zhang2013review, liu2021emerging}. Over the past decade, MLL has been successfully applied to various real-world applications, including image annotation \cite{wang2009multi}, text classification \cite{liu2017deep}, and facial expression recognition \cite{chen2020label}.

Compared with multi-class-single-label learning, where each example is associated with a unique label, MLL involves instances that are assigned multiple labels. As the number of examples or categories is large, accurately annotating each label of an example becomes exceedingly challenging. To address the high annotation cost, single-positive multi-label learning (SPMLL) has been proposed \cite{cole2021multi, xu2022one}, where each training example is annotated with only one positive label. Moreover, since many examples in multi-class datasets, such as ImageNet \cite{YunOHHCC21}, contain multiple categories but are annotated with a single label, employing SPMLL allows for the derivation of multi-label predictors from existing numerous multi-class datasets, thereby expanding the applicability of MLL.

To address the issue that model tends to predict all labels as positive if trained with only positive labels, existing SPMLL methods typically assign pseudo-labels to unannotated labels. Cole et al. updates the pseudo-labels as learnable parameters with a regularization to constrain the number of expected positive labels \cite{cole2021multi}. Xu et al. recovers latent soft pseudo-labels by employing variational label enhancement \cite{xu2022one}. Zhou et al. adopts asymmetric-tolerance strategies to update pseudo-labels cooperating with an entropy-maximization loss \cite{zhou2022acknowledging}. Xie et al. utilizes contrastive learning to learn the manifold structure information and updates the pseudo-labels with a threshold \cite{xielabel}.

These approaches rely on a crucial assumption that prior probabilities of all classes are identical. However, in real-world scenarios, the class-prior of each category may differ significantly. This unrealistic assumption will introduce severe biases into the pseudo-labels, further impacting the training of the model supervised by the inaccurate pseudo-labels. 
As a result, the learned model could not perform as well as expected.

Motivated by the above consideration, we propose a novel framework named {\proposed}, i.e., Class-pRiors Induced Single-Positive multi-label learning.
Specifically, a class-priors estimator is derived, which determines an optimal threshold by estimating the ratio between the fraction of positive labeled samples and the total number of samples receiving scores above the threshold. The estimated class-priors can be theoretically guaranteed to converge to the ground-truth class-priors. 
In addition, based on the estimated class-priors, an unbiased risk estimator for classification is derived, which guarantees the learning consistency \cite{Foundations} and ensures that the obtained risk minimizer would approximately converge to the optimal risk minimizer on fully supervised data. Our contributions can be summarized as follows:

\begin{itemize}[leftmargin=*]
  \item Practically, for the first time, we propose a novel framework for SPMLL named {\proposed}, which estimates the class-priors and then an unbiased risk estimator is derived based on the estimated class-priors, addressing the unrealistic assumption of identical class-priors for all classes.
  \item Theoretically, the estimated class-priors can be guaranteed to converge to the ground-truth class-priors. Additionally, we prove that the risk minimizer corresponding to the proposed risk estimator can be guaranteed to approximately converge to the optimal risk minimizer on fully supervised data.
\end{itemize}

Experiments on four multi-label image classification (MLIC) datasets and six MLL datasets show the effectiveness of our methods over several existing SPMLL approaches.

\section{Related Work}

Multi-label learning is a supervised machine learning technique where an instance is associated with multiple labels simultaneously. The study of label correlations in multi-label learning has been extensive, and these correlations can be categorized into first-order, second-order, and high-order correlations. First-order correlations involve adapting binary classification algorithms for multi-label learning, such as treating each label as an independent binary classification problem \cite{boutell2004learning, read2011classifier}. Second-order correlations model pairwise relationships between labels \cite{elisseeff2001kernel, furnkranz2008multilabel}. High-order correlations take into account the relationships among multiple labels, such as employing graph convolutional neural networks to extract correlation information among all label nodes \cite{chen2019multi}. Furthermore, there has been an increasing interest in utilizing label-specific features, which are tailored to capture the attributes of a specific label and enhance the performance of the models \cite{yu2021multi, hang2021collaborative}.

In practice, accurately annotating each label for every instance in multi-label learning is unfeasible due to the immense scale of the output space. Consequently, multi-label learning with missing labels (MLML) has been introduced \cite{sun2010multi}. MLML methods primarily rely on low-rank, embedding, and graph-based models. The presence of label correlations implies a low-rank output space \cite{liu2021emerging}, which has been extensively employed to fill in the missing entries in a label matrix \cite{xu2013speedup, yu2014large, xu2016robust}. Another widespread approach is based on embedding techniques that map label vectors to a low-dimensional space, where features and labels are jointly embedded to exploit the complementarity between the feature and label spaces \cite{yeh2017learning, wang2019robust}. Additionally, graph-based models are prevalent solutions for MLML, constructing a label-specific graph for each label from a feature-induced similarity graph and incorporating manifold regularization into the empirical risk minimization framework \cite{sun2010multi, wu2014multi}.

In SPMLL, a specific case of multi-label learning with incomplete labels, only one of the multiple positive labels is observed. The initial work treats all unannotated labels as negative and updates the pseudo-labels as learnable parameters, applying a regularization to constrain the number of expected positive labels \cite{cole2021multi}. A label enhancement process is used to recover latent soft labels and train the multi-label classifier \cite{xu2022one}. The introduction of an asymmetric pseudo-label approach utilizes asymmetric-tolerance strategies for pseudo-labels, along with an entropy-maximization loss \cite{zhou2022acknowledging}. Additionally, Xie et al. proposes a label-aware global consistency regularization method, leveraging the manifold structure information learned from contrastive learning to update pseudo-labels \cite{xielabel}.

\section{Preliminaries}

Multi-label learning (MLL) aims to train a model on the examples that are associated with multiple labels and obtain a predictive model that is able to predict the relevant labels for an unknown instance accurately. 
Let $ \mathcal{X}=\mathbb{R}^q $ denote the instance space and $ \mathcal{Y}=\{0, 1\}^c $ denote the label space with $ c $  classes. A MLL training set is denoted by $ \mathcal{D}=\{(\bm x_i, \bm y_i) \vert 1\leq i\leq n\} $ where $ \bm x_i \in\mathcal{X} $ is a $ q $-dimensional instance and $ \bm y_i \in\mathcal{Y} $ is its corresponding labels. Here, $ \bm y_i = [y_i^1, y_i^2, \dots, y_i^c] $ where $ y_i^j=1 $ indicates that the $ j $-th label is a relevant label associated with $ \bm x_i $ and $ y_i^j=0 $ indicates that the $ j $-th label is irrelevant to $ \bm x_i $.
For single-positive multi-label learning (SPMLL), each instance is annotated with only one positive label. Given the SPMLL training set $ \tilde{\mathcal{D}}=\{(\bm x_i, \gamma_i) \vert 1\leq i\leq n\} $ where $ \bm x_i \in\mathcal{X} $ is a $ q $-dimensional instance and $ \gamma_i \in \{1, 2, \dots, c \} $ denotes the only observed single positive label of $ \bm x_i $. For each SPMLL training example $(\bm{x}_i, \gamma_i)$, we use the observed single-positive label vector $\bm{l}_i=[l_i^{1},l_i^{2},\ldots,l_i^{c}]^\top \in \{0,1\}^c$ to represent whether $j$-th label is the observed positive label, i.e.,  $l_i^{j} = 1$ if $j=\gamma_i $, otherwise $l_i^{j} = 0$.
The task of SPMLL is to induce a multi-label classifier in the hypothesis space $ h \in \mathcal{H}:\mathcal{X}\mapsto\mathcal{Y} $ that minimizes the following classification risk:
\begin{equation}
    \mathcal{R}(h)=\mathbb{E}_{(\bm x,\bm y)\sim p(\bm x, \bm y)}\left [ \mathcal{L}(h(\bm x), \bm y) \right ],
\end{equation}
where $ \mathcal{L}: \mathcal{X}\times\mathcal{Y}\mapsto\mathbb{R}^{+} $ is a multi-label loss function that measures the accuracy of the model in fitting the data. Note that a method is risk-consistent if the method possesses a classification risk estimator that is equivalent to $ \mathcal{R}(h) $ given the same classifier \cite{Foundations}.

\section{The Proposed Method}

\subsection{The {\proposed} Algorithm}

\begin{algorithm}[t]
  \caption{{\proposed} Algorithm}
  \label{alg:class_prior_estimator}
  \begin{algorithmic}[1]
  \REQUIRE The SPMLL training set $ \tilde{\mathcal{D}}=\{(\bm x_i, \gamma_i) \vert 1\leq i\leq n\} $, the multi-label classifier $ f $, the number of epoch $ T $, hyperparameters $ \delta>0, \tau\leq 1 $ and $ \lambda $ ;
  \STATE Warm up the model $ f $ with AN strategy \cite{cole2021multi} (Assume the unobserved labels as negative ones).
  \FOR{$t=1$ {\bfseries to} $T$}
  \FOR{$j=1$ {\bfseries to} $c$}
  \STATE Extract the positive-labeled samples set $ \mathcal S_{L_j} = \{ \bm x_i: l_i^j = 1, 1\leq i \leq n \} $.
  \STATE Estimate $ \hat{q}_j(z) = \frac{1}{n}\sum_{i=1}^{n}\mathbf{1}(f^j(\bm x_i)\geq z)$ and $ \hat{q}_j^p(z) = \frac{1}{n_j^p}\sum_{\bm x \in \mathcal S_{L_j}}\mathbf{1}(f^j(\bm x)\geq z) $ for all $ z\in [0,1] $. 
  \STATE Estimate the class-prior of $ j $-th label by $ \hat\pi_j = \frac{\hat q_j(\hat z)}{\hat q_j^p(\hat z)} $ with the threshold induced by Eq. (\ref{eq:threshold}).
  \ENDFOR
  \STATE Update the model $ f $ by forward computation and back-propagation by Eq. (\ref{eq:empirical_risk_estimator_bias}) using the estimated class-priors.
  \ENDFOR
  \ENSURE The predictive model $ f $.
  \end{algorithmic}
\end{algorithm}

In this section, we introduce our novel framework, {\proposed}, i.e., Class-pRiors Induced Single-Positive multi-label learning. This framework alternates between estimating class-priors and optimizing an unbiased risk estimator under the guidance of the estimated class-priors.

Firstly, we introduce the class-priors estimator for SPMLL, leveraging the blackbox classifier $ f $ to estimate the class-prior of each label. The class-priors estimator exploits the classifier $ f $ to give each input a score, indicating the likelihood of it belonging to a positive sample of $ j $-th label. 

Motivated by the definition of top bin in learning from positive and unlabeled data (PU learning)  \cite{GargWSBL21}, for a given probability density function $ p(\bm x) $ and a classifier $ f $, define the threshold cumulative density function $ q_j(z)=\int_{S_z}p(\bm x) d\bm x $ where $ S_z=\{\bm x \in\mathcal{X}:f^j(\bm x)\geq z\} $ for all $ z\in[0, 1] $. $ q_j(z) $ captures the cumulative density of the feature points which are assigned a value larger than a threshold $ z $ by the classifier of the $ j $-th label. We now define an empirical estimator of $ q_j(z) $ as $ \hat q_j(z)=\frac{1}{n}\sum_{i=1}^{n}\mathbf{1}(f^j(\bm x_i)\geq z) $ where $ \mathbf{1}(\cdot) $ is the indicator function. For each probability density function $ p_j^p=p(\bm x\vert y_j=1) $ and $ p_j^n=p(\bm x\vert y_j=0) $, we define $ q_j^p =\int_{S_z} p(\bm x \vert y_j=1) d\bm x $ and $ q_j^n =\int_{S_z}p(\bm x \vert y_j=0) d\bm x $ respectively.

The steps involved in the procedure are as follows: Firstly, for each label, we extract a positive-labeled samples set $ \mathcal S_{L_j} = \{ \bm x_i: l_i^j = 1, 1\leq i \leq n \} $ from the entire dataset. Next, with $ \mathcal S_{L_j} $, 
we estimate the fraction of the total number of samples that receive scores above the threshold $ \hat q_j(z) = \frac{1}{n}\sum_{i=1}^{n}\mathbf{1}(f^j(\bm x_i)\geq z) $ and that of positive labeled samples receiving scores above the threshold $ \hat q_j^p(z) = \frac{1}{n_j^p}\sum_{\bm x \in \mathcal S_{L_j}}\mathbf{1}(f^j(\bm x)\geq z) $ for all $ z\in[0, 1] $,
where $ n_j^p = \vert \mathcal S_{L_j} \vert $ is the cardinality of the positive-labeled samples set of $ j $-th label.
Finally, the class-prior of $ j $-th label is estimated by $ \frac{\hat q_j(\hat z)}{\hat q_j^p(\hat z)} $ at $ \hat z $ that minimizes the upper confidence bound defined in Theorem \ref{thm:cpe}.

\begin{theorem}\label{thm:cpe}
  
  Define $ z^\star = \arg\min_{z\in[0,1]} q_j^n(z) / q_j^p(z) $, for every $ 0<\delta<1 $, define $ \hat{z} = \arg\min_{z\in [0,1]} \left( \frac{\hat q_j(z)}{\hat q_j^p(z)} + 
  \frac{1+\tau}{\hat q_j^p(z)}\left( \sqrt{\frac{\log(4/\delta)}{2n}} + \sqrt{\frac{\log(4/\delta)}{2n_j^p}} \right) \right) $. Assume $ n_j^p \geq 2\frac{\log4/\delta}{q_j^p(z^\star)} $, the estimated class-prior $ \hat\pi_j=\frac{\hat q_j(\hat z)}{\hat q_j^p(\hat z)} $ satisfies with probability at least $ 1-\delta $:
  \begin{equation*}
    \begin{aligned}
      \pi_j - \frac{c_1}{q_j^p(z^\star)}\left( \sqrt{\frac{\log(4/\delta)}{2n}} + \sqrt{\frac{\log(4/\delta)}{2n_p}} \right) \leq & \hat\pi_j \leq \pi_j + (1-\pi_j)\frac{q_j^n(z^\star)}{q_j^p(z^\star)} \\
      & + \frac{c_2}{q_j^p(z^\star)}\left( \sqrt{\frac{\log(4/\delta)}{2n}} + \sqrt{\frac{\log(4/\delta)}{2n_p}} \right),
    \end{aligned}
  \end{equation*}
\end{theorem}
where $ c_1, c_2 \geq 0 $ are constants and $ \tau $ is a fixed parameter ranging in $ (0, 1) $. The proof can be found in Appendix \ref{app:proof_cpe}. 
Theorem \ref{thm:cpe} provides a principle for finding the optimal threshold. Under the condition that the threshold $ \hat z $ satisfies:
\begin{equation}\label{eq:threshold}
  \begin{aligned}
    \hat{z} = \arg\min_{z\in [0,1]} \Bigg( \frac{\hat q_j(z)}{\hat q_j^p(z)} + 
    \frac{1+\tau}{\hat q_j^p(z)}\bigg( \sqrt{\frac{\log(4/\delta)}{2n}} 
    + \sqrt{\frac{\log(4/\delta)}{2n_j^p}} \bigg) \Bigg),
  \end{aligned}
\end{equation}
with enough training samples, the estimated class-prior $ \hat\pi_j $  of $ j $-th category will converge to the ground-truth class-prior with a tolerable error bounded by $ \frac{q_j^n(z^\star)}{q_j^p(z^\star)} $  in the upper bound, which is determined by the current classifier's capability. If the classifier can more accurately classify negative samples, making $q_j^n(z^\star)$ smaller, and simultaneously better classify positive samples, making $q_j^p(z^\star)$ larger, then the overall error term $\frac{q_j^n(z^\star)}{q_j^p(z^\star)}$ will become smaller.
Practically, to determine the optimal threshold in Eq. (\ref{eq:threshold}), we conduct an exhaustive search across the set of outputs generated by the function $ f^j $ for each class. The details can be found in Appendix \ref{app:threshold}.


After obtaining an accurate estimate of class-prior for each category, we proceed to utilize these estimates as a form of supervision to guide the training of our model. 
Firstly, the classification risk $ \mathcal{R}(f) $ on fully supervised information can be written as \footnote{The datail is provided in Appendix \ref{app:proof_risk}.}:
\begin{equation}\label{eq:full_label_loss_func}
  \begin{aligned}
    \mathcal{R}(f) &= \mathbb{E}_{(\bm x, \bm y)\sim p(\bm x, \bm y)} \left[ \mathcal{L}(f(\bm x), \bm y) \right] 
    = \sum_{\bm y} p(\bm y) \mathbb{E}_{\bm x \sim p(\bm x \vert \bm y)} \left[ \mathcal{L}(f(\bm x), \bm y) \right].
  \end{aligned}
\end{equation}

In Eq. (\ref{eq:full_label_loss_func}), the loss function $ \mathcal{L}(f(\bm x), \bm y) $ is calculated for each label separately, which is a commonly used approach in multi-label learning:
\begin{equation}\label{eq:loss_func}
  \begin{aligned}
    \mathcal{L}(f(\bm x), \bm y) = &\sum_{j=1}^c y_j\ell(f^j(\bm x), 1) + (1 - y_j)\ell(f^j(\bm x), 0).
  \end{aligned}
\end{equation}

By substituting Eq. (\ref{eq:loss_func}) into Eq. (\ref{eq:full_label_loss_func}), the classification risk $ \mathcal{R}(f) $ can be written as follows with the absolute loss function\footnote{The detail is provided in Appendix \ref{app:proof_risk2}.}:
\begin{equation}\label{eq:full_label_loss_func2}
  \begin{aligned}
    \mathcal{R}(f) =& \sum_{j=1}^c 2p(y_j = 1) \mathbb{E}_{\bm x \sim p(\bm x \vert y_j = 1)}\left[ 1 - f^j(\bm x) \right]
    + \left(\mathbb{E}_{\bm x \sim p(\bm x)}\left[f^j(\bm x)\right] - p(y_j = 1)\right).
  \end{aligned}
\end{equation}
The rewritten classification risk comprises two distinct components. The first component computes the risk solely for the positively labeled samples, and the second component leverages the unlabeled data to estimate difference between the expected output of the model $f$ and the class-prior $\pi_j=p(y_j=1)$ to align the expected class-prior outputted by model with the ground-truth class-prior.

During the training process, the prediction of model can be unstable due to insufficiently labeled data. This instability may cause a large divergence between the expected class-prior $ \mathbb{E}[f^j(\bm x)] $ and the ground-truth class-prior $\pi_j$, even leading to a situation where the difference between $ \mathbb{E}[f^j(\bm x)] $ and $\pi_j$ turns negative \cite{zhao2022dist}. To ensure non-negativity of the classification risk and the alignment of class-priors, absolute function is added to the second term. Then the risk estimator can be written as:
\begin{equation}
  \begin{aligned}
    \mathcal{R}_{sp}(f) =& \sum_{j=1}^c 2\pi_j \mathbb{E}_{\bm x \sim p(\bm x \vert y_j = 1)}\left[ 1 - f^j(\bm x) \right]
    + \biggl| \mathbb{E}_{\bm x \sim p(\bm x)}\left[f^j(\bm x)\right] - \pi_j \biggr|.
  \end{aligned}
\end{equation}

Therefore, we could express the empirical risk estimator via:
\begin{equation}\label{eq:empirical_risk_estimator}
  \begin{aligned}
    \widehat{\mathcal{R}}_{sp}(f) =& \sum_{j=1}^c \frac{2\pi_j}{\vert \mathcal S_{L_j} \vert} \sum_{\bm x \in \mathcal S_{L_j}} \left(1 - f^j(\bm x)\right)
    + \biggl| \frac{1}{n} \sum_{\bm x \in \tilde{\mathcal{D}}}\left(f^j(\bm x) - \pi_j\right)\biggr|.
  \end{aligned}
\end{equation}

In MLL datasets, where the number of negative samples for each label significantly exceeds that of positive samples, there is a tendency for the decision boundary to be biased towards the center of  positive samples,  especially for rare classes. This bias is further exacerbated in SPMLL due to the common strategy of assuming unobserved labels as negative \cite{cole2021multi, xu2022one, xielabel} to warm up the model. To alleviate the issue, we propose a modification of Eq. (\ref{eq:empirical_risk_estimator}):
\begin{equation}\label{eq:empirical_risk_estimator_bias}
  \begin{aligned}
    \widehat{\mathcal{R}}_{sp}(f) =& \sum_{j=1}^c \frac{2\pi_j}{\vert \mathcal S_{L_j} \vert} \sum_{\bm x \in \mathcal S_{L_j}} \left(1 - \frac{e^{g^j(\bm x)-\lambda b^j}}{e^{g^j(\bm x)-\lambda b^j}+1}\right)
    + \biggl| \frac{1}{n} \sum_{\bm x \in \tilde{\mathcal{D}}}\left(f^j(\bm x) - \pi_j\right)\biggr|,
  \end{aligned}
\end{equation}
where  $ b^j=1-\pi^j $, $\lambda$ is a hyper-parameter, $g^j(\bm x)$ denotes the logit of $j$-th label outputted by the model for instance $\bm x$ and $f^j(\bm x)=\sigma(g^j(\bm x))$ with $\sigma(\cdot)$ representing the sigmoid function. The model tends to produce a lower probility for $ j $-th label when $ \pi^j $ is smaller, then, Eq. (\ref{eq:empirical_risk_estimator_bias}) introduces a larger bias $ \lambda b^j $ to encourage the model to yield a higher output for positive samples, which modifies the decision boundary towards the center of the negative samples.

The algorithmic description of {\proposed} is shown in Algorithm \ref{alg:class_prior_estimator}.

\subsection{Estimation Error Bound}
In this subsection, an estimation error bound is established for Eq. (\ref{eq:empirical_risk_estimator}) to demonstrate its learning consistency. 
Firstly, we define the function spaces as:
\begin{equation*}
  \begin{aligned}
    \mathcal{G}_{sp}^L &= \Big\{(\bm x, \bm l)\mapsto\sum_{j=1}^c 2\pi_j l_j \left(1 - f^j(\bm x)\right) \vert f \in \mathcal{F} \Big\},
    \mathcal{G}_{sp}^U &= \Big\{\bm x\mapsto\sum_{j=1}^c \left(f^j(\bm x) - \pi_j \right) \vert f \in \mathcal{F} \Big\},
  \end{aligned}
\end{equation*}
and denote the expected Rademacher complexity \cite{Foundations} of the function spaces as:
\begin{equation*}
  \begin{aligned}
    \widetilde{\mathfrak{R}}_n\left(\mathcal{G}_{sp}^L\right)&=\mathbb{E}_{\bm{x}, \bm y, \bm{\sigma}}\left[\sup _{g \in \mathcal{G}_{sp}^L} \sum_{i=1}^n \sigma_i g\left(\bm{x}_i, \bm l_i\right)\right],
    \widetilde{\mathfrak{R}}_n\left(\mathcal{G}_{sp}^U \right)&=\mathbb{E}_{\bm{x}, \bm{\sigma}}\left[\sup _{g \in \mathcal{G}_{sp}^U} \sum_{i=1}^n \sigma_i g\left(\bm{x}_i\right)\right],
  \end{aligned}
\end{equation*}
where $ \bm\sigma=\left\{ \sigma_1, \sigma_2, \cdots, \sigma_n \right\} $ is $ n $ Rademacher variables with $ \sigma_i $ independently uniform variable taking value in $ \{+1,-1\} $. 
Then we have:
\begin{theorem}\label{theorem:theorem1}
  Assume the loss function $ \mathcal{L}_{sp}^L = \sum_{j=1}^c 2\pi_j l_j \left(1 - f^j(\bm x)\right) $ and $ \mathcal{L}_{sp}^U = \sum_{j=1}^c \left(f^j(\bm x) - \pi_j \right) $  could be bounded by $ M $, i.e., $ M=\sup_{\bm x\in \mathcal{X}, f\in\mathcal{F}, \bm y\in \mathcal{Y}} \max(\mathcal{L}_{sp}^L(f(\bm x), \bm l), \mathcal{L}_{sp}^U(f(\bm x))) $, with probability at least $ 1-\delta $, we have:
  \begin{equation*}
    \begin{aligned}
      \mathcal{R}(\hat{f}_{sp}) - \mathcal{R}(f^\star) \leq 
      &\frac{4\sqrt{2}\rho}{C} \sum_{j=1}^{c}\mathfrak{R}_n(\mathcal{H}_j) + \frac{M}{\min_{j}\vert \mathcal S_{L_j} \vert} \sqrt{\frac{\log\frac{4}{\delta}}{2n}} 
      + 4\sqrt{2} \sum_{j=1}^{c}\mathfrak{R}_n(\mathcal{H}_j) + M \sqrt{\frac{\log\frac{4}{\delta}}{2n}}.
    \end{aligned}
  \end{equation*}
\end{theorem}
where $ C $ is a constant, $ \hat{f}_{sp}=\min_{f\in\mathcal{F}}\widehat{\mathcal{R}}_{sp}(f) $, $ f^\star=\min_{f\in\mathcal{F}}\mathcal{R}(f) $ are the empirical risk minimizer and the true risk minimizer respectively and $ \rho = \max_j 2\pi_j $, $ \mathcal{H}_j = \left\{ h: \bm x \mapsto f^j(\bm x) \vert f\in \mathcal{F} \right\} $ and $ \mathfrak{R}_n\left(\mathcal{H}_j\right)=\mathbb{E}_{p(\bm{x})} \mathbb{E}_{\bm{\sigma}}\left[\sup _{h \in \mathcal{H}_j} \frac{1}{n} \sum_{i=1}^n h\left(\bm{x}_i\right)\right] $. The proof can be found in Appendix \ref{app:proof_error_bound}. 

Theorem \ref{theorem:theorem1} shows that, as $ n\rightarrow\infty $, $ \hat{f}_{sp} $ would converge to $ f^\star $ with an intrinsic error quantified by the Rademacher complexity terms, reflecting the complexity of the hypothesis space. Note that the error is a fundamental aspect of the learning problem and remains even in a fully supervised scenario \cite{Foundations}.

\section{Experiments}
\subsection{Experimental Configurations}
{\bf Datasets. } In the experimental section, our proposed method is evaluated on four large-scale multi-label image classification (MLIC) datasets and six widely-used multi-label learning (MLL) datasets.
The four MLIC datasets include \texttt{PSACAL VOC 2021 (VOC)} \cite{EveringhamGWWZ10}, \texttt{MS-COCO 2014 (COCO)} \cite{LinMBHPRDZ14}, \texttt{NUS-WIDE (NUS)} \cite{ChuaTHLLZ09}, and \texttt{CUB-200 2011 (CUB)} \cite{wah2011caltech}; the MLL datasets cover a wide range of scenarios with various multi-label characteristics. For each MLIC dataset, $  20 \% $  of the training set is withheld for validation. Each MLL dataset is partitioned into train/validation/test sets at a ratio of $ 80\%/10\%/10\% $. One positive label is randomly selected for each training instance, while the validation and test sets remain fully labeled. Detailed information regarding these datasets can be found in Appendix \ref{DS_details}. \textit{Mean average precision (mAP)} is utilized for the four MLIC datasets \cite{cole2021multi, xielabel, zhou2022acknowledging} and five popular multi-label metrics are adopted for the MLL datasets including \textit{Ranking loss, Hamming loss, One-error, Coverage} and \textit{Average precision} \cite{xu2022one}.

{\bf Comparing methods. } In this paper, {\proposed} is compared against several state-of-the-art SPMLL approaches including: 
\begin{enumerate*}[1)]
    \item {\an} \cite{cole2021multi} assumes that the unannotated labels are negative and uses binary cross entropy loss for training.
    \item {\anls} \cite{cole2021multi} assumes that the unannotated labels are negative and reduces the impact of the false negative labels by label smoothing.
    \item {\wan} \cite{cole2021multi} introduces a weight parameter to down-weight losses in relation to negative labels.
    \item {\epr} \cite{cole2021multi} utilizes a regularization to constrain the number of predicted positive labels.
    \item {\role} \cite{cole2021multi} online estimates the unannotated labels as learnable parameters based on {\epr} with the trick of linear initial.
    \item {\emm} \cite{zhou2022acknowledging} reduces the effect of the incorrect labels by the entropy-maximization loss.
    \item {\emapl}  \cite{zhou2022acknowledging} adopts asymmetric-tolerance pseudo-label strategies cooperating with entropy-maximization loss and then more precise supervision can be provided.
    \item {\lagc} \cite{xielabel} designs a label-aware global consistency regularization to recover the pseudo-labels leveraging the manifold structure information learned by contrastive learning with data augmentation techniques.
    \item {\smilee} \cite{xu2022one} recovers the latent soft labels in a label enhancement process to train the multi-label classifier with the proposed consistency risk minimizer.
    \item {\smib} \citep{liu2023revisiting} generates pseudo-labels based on estimated mutual information which can simultaneously train the model and update pseudo-labels in a label enhancement process.
  \end{enumerate*}
Additionally, since the SPMLL problem is an extreme case of the MLML problem, we employ a state-of-the-art MLML methods as comparative methods:
\begin{enumerate*}[1)]
    \item {\lln} \cite{kim2022large} treats unobserved labels as noisy labels and dynamically adjusts the threshold to reject or correct samples with a large loss, including three variants {\llr}, {\llct} and {\llcp}.
    \item {\boost} \cite{KimKJSA023} apply a BoostLU function to the CAM output of the model to boost the scores of the highlighted regions. It is integrated with {\lln}.
\end{enumerate*}
The implementation details are provided in Appendix \ref{implementation}.

\begin{table*}[t]
  \caption{Predictive performance of each comparing method on four MLIC datasets in terms of \textit{mean average precision (mAP)} (mean $ \pm $ std). The best performance is highlighted in bold (the larger the better).}
  \label{mAP}
  \centering
  \resizebox{0.9\linewidth}{!}{
  \begin{tabular}{ccccc}
  \toprule
                              & VOC          & COCO         & NUS          & CUB          \\ \midrule
  {\an}                       & 85.546$\pm$0.294 & 64.326$\pm$0.204 & 42.494$\pm$0.338 & 18.656$\pm$0.090 \\
  {\anls}                     & 87.548$\pm$0.137 & 67.074$\pm$0.196 & 43.616$\pm$0.342 & 16.446$\pm$0.269 \\
  {\wan}                      & 87.138$\pm$0.240 & 65.552$\pm$0.171 & 45.785$\pm$0.192 & 14.622$\pm$1.300 \\
  {\epr}                      & 85.228$\pm$0.444 & 63.604$\pm$0.249 & 45.240$\pm$0.338 & 19.842$\pm$0.423 \\
  {\role}                     & 88.088$\pm$0.167 & 67.022$\pm$0.141 & 41.949$\pm$0.205 & 14.798$\pm$0.613 \\
  {\emm}                      & 88.674$\pm$0.077 & 70.636$\pm$0.094 & 47.254$\pm$0.297 & 20.692$\pm$0.527 \\
  {\emapl}                    & 88.860$\pm$0.080 & 70.758$\pm$0.215 & 47.778$\pm$0.181 & 21.202$\pm$0.792 \\
  {\smilee}                   & 87.314$\pm$0.150 & 70.431$\pm$0.213 & 47.241$\pm$0.172 & 18.611$\pm$0.144 \\
  {\lagc}                     & 88.021$\pm$0.121 & 70.422$\pm$0.062 & 46.211$\pm$0.155 & 21.840$\pm$0.237 \\ 
  {\smib}                     & 89.199$\pm$0.157 & 72.920$\pm$0.255 & 48.743$\pm$0.428 & 21.890$\pm$0.347 \\   \midrule
  {\llr}                      & 87.784$\pm$0.063 & 70.078$\pm$0.008 & 48.048$\pm$0.074 & 18.966$\pm$0.022 \\
  {\llcp}                     & 87.466$\pm$0.031 & 70.460$\pm$0.032 & 48.000$\pm$0.077 & 19.310$\pm$0.164 \\
  {\llct}                     & 87.054$\pm$0.214 & 70.384$\pm$0.058 & 47.930$\pm$0.010 & 19.012$\pm$0.097 \\
  {\boost}+{\llr}             & 89.224$\pm$0.017 & 73.272$\pm$0.006 & 49.590$\pm$0.021 & 19.136$\pm$0.009 \\
  {\boost}+{\llcp}            & 88.358$\pm$0.212 & 70.820$\pm$0.030 & 47.810$\pm$0.166 & 18.166$\pm$0.063 \\
  {\boost}+{\llct}            & 88.528$\pm$0.053 & 71.742$\pm$0.006 & 48.216$\pm$0.021 & 17.952$\pm$0.007 \\ \midrule
  {\proposed}                 & \textbf{89.931$\pm$0.014} & \textbf{74.913$\pm$0.235} & \textbf{50.045$\pm$0.112} & \textbf{22.150$\pm$0.028} \\ \bottomrule
  \end{tabular}
  }
\end{table*}

\begin{table*}[t]
    \caption{Predictive performance of each comparing method on MLL datasets in terms of \textit{Ranking loss} (mean $ \pm $ std). The best performance is highlighted in bold (the smaller the better).}
        \label{ranking}
        \centering
        \resizebox{0.9\linewidth}{!}
        {
        \begin{tabular}{ccccccc}
        \toprule
                          & Image       & Scene       & Yeast       & Corel5k    & Mirflickr   & Delicious   \\ \midrule
        {\an}       & 0.432$\pm$0.067          & 0.321$\pm$0.113          & 0.383$\pm$0.066          & 0.140$\pm$0.000          & 0.125$\pm$0.002          & 0.131$\pm$0.000          \\
        {\anls}    & 0.378$\pm$0.041          & 0.246$\pm$0.064          & 0.365$\pm$0.031          & 0.186$\pm$0.003          & 0.163$\pm$0.006          & 0.213$\pm$0.007          \\
        {\wan}     & 0.354$\pm$0.051          & 0.216$\pm$0.023          & 0.212$\pm$0.021          & 0.129$\pm$0.000          & 0.121$\pm$0.002          & 0.126$\pm$0.000          \\
        {\epr}     & 0.401$\pm$0.053          & 0.291$\pm$0.056          & 0.208$\pm$0.010          & 0.139$\pm$0.000          & 0.119$\pm$0.001          & 0.126$\pm$0.000          \\
        {\role}    & 0.340$\pm$0.059          & 0.174$\pm$0.028          & 0.213$\pm$0.017          & 0.259$\pm$0.004          & 0.182$\pm$0.014          & 0.336$\pm$0.007          \\
        {\emm}      & 0.471$\pm$0.044          & 0.322$\pm$0.115          & 0.261$\pm$0.030          & 0.155$\pm$0.002          & 0.134$\pm$0.004          & 0.164$\pm$0.001          \\
        {\emapl} & 0.508$\pm$0.028          & 0.420$\pm$0.069          & 0.245$\pm$0.026          & 0.135$\pm$0.001          & 0.138$\pm$0.003          & 0.163$\pm$0.003          \\
        {\smilee}   & 0.260$\pm$0.020          & 0.161$\pm$0.045          & 0.167$\pm$0.002          & 0.125$\pm$0.003          & 0.120$\pm$0.002          & 0.126$\pm$0.000          \\ 
        {\smib}     & 0.251$\pm$0.003              & 0.163$\pm$0.001              & 0.167$\pm$0.003              & 0.135$\pm$0.002              & 0.137$\pm$0.000              & 0.153$\pm$0.001             \\ \midrule
        {\llr}  & 0.346$\pm$0.072 & 0.155$\pm$0.021 & 0.227$\pm$0.001 & 0.114$\pm$0.001 & 0.123$\pm$0.003 & 0.129$\pm$0.002 \\
        {\llcp} & 0.329$\pm$0.041 & 0.148$\pm$0.017 & 0.215$\pm$0.000 & 0.114$\pm$0.003 & 0.124$\pm$0.003 & 0.160$\pm$0.001 \\
        {\llct} & 0.327$\pm$0.019 & 0.180$\pm$0.038 & 0.238$\pm$0.001 & 0.115$\pm$0.001 & 0.124$\pm$0.002 & 0.160$\pm$0.000 \\ \midrule
        {\proposed}    & \textbf{0.164$\pm$0.027} & \textbf{0.112$\pm$0.021} & \textbf{0.164$\pm$0.001} & \textbf{0.113$\pm$0.001} & \textbf{0.118$\pm$0.001} & \textbf{0.122$\pm$0.000}\\ \bottomrule
        \end{tabular}
        }
\end{table*}

\begin{table*}[t]
    \caption{Predictive performance of each comparing method on MLL datasets in terms of \textit{Average Precision} (mean $ \pm $ std). The best performance is highlighted in bold (the larger the better).}
        \label{ap}
        \centering
        \resizebox{0.9\linewidth}{!}
        {
        \begin{tabular}{ccccccc}
        \toprule
                      & Image       & Scene       & Yeast       & Corel5k    & Mirflickr   & Delicious            \\ \midrule
        {\an}           & 0.534$\pm$0.061          & 0.580$\pm$0.104          & 0.531$\pm$0.079          & 0.217$\pm$0.003          & 0.615$\pm$0.004          & 0.317$\pm$0.002          \\
        {\anls}          & 0.574$\pm$0.037          & 0.631$\pm$0.072          & 0.538$\pm$0.044          & 0.230$\pm$0.002          & 0.587$\pm$0.006          & 0.261$\pm$0.006          \\
        {\wan}           & 0.576$\pm$0.041          & 0.661$\pm$0.033          & 0.698$\pm$0.017          & 0.241$\pm$0.002          & 0.621$\pm$0.004          & 0.315$\pm$0.000          \\
        {\epr}          & 0.539$\pm$0.028          & 0.597$\pm$0.062          & 0.710$\pm$0.008          & 0.214$\pm$0.001          & 0.628$\pm$0.003          & 0.314$\pm$0.000          \\
        {\role}          & 0.606$\pm$0.041          & 0.700$\pm$0.040          & 0.711$\pm$0.013          & 0.203$\pm$0.003          & 0.516$\pm$0.027          & 0.130$\pm$0.003          \\
        {\emm}            & 0.486$\pm$0.031          & 0.549$\pm$0.103          & 0.642$\pm$0.029          & 0.294$\pm$0.002          & 0.614$\pm$0.003          & 0.293$\pm$0.001          \\
        {\emapl}       & 0.467$\pm$0.026          & 0.448$\pm$0.049          & 0.654$\pm$0.040          & 0.275$\pm$0.003          & 0.589$\pm$0.007          & 0.311$\pm$0.001          \\
        {\smilee}         & 0.670$\pm$0.021          & 0.722$\pm$0.071          & 0.751$\pm$0.004          & 0.295$\pm$0.004          & \textbf{0.629$\pm$0.003} & 0.318$\pm$0.001          \\
        {\smib}           & 0.675$\pm$0.003             & 0.711$\pm$0.006               & 0.712$\pm$0.002             & 0.257$\pm$0.005               & 0.628$\pm$0.001             & 0.257$\pm$0.004    \\ \midrule
        {\llr}  & 0.605$\pm$0.058 & 0.714$\pm$0.035 & 0.658$\pm$0.006 & 0.268$\pm$0.002          & 0.625$\pm$0.001 & 0.296$\pm$0.004 \\
        {\llcp} & 0.595$\pm$0.031 & 0.735$\pm$0.028 & 0.700$\pm$0.000 & 0.259$\pm$0.004 & 0.621$\pm$0.007 & 0.251$\pm$0.007          \\
        {\llct} & 0.600$\pm$0.012 & 0.669$\pm$0.052 & 0.629$\pm$0.007 & 0.258$\pm$0.004          & 0.619$\pm$0.004 & 0.253$\pm$0.004          \\ \midrule
        {\proposed} & \textbf{0.749$\pm$0.037} & \textbf{0.795$\pm$0.031} & \textbf{0.758$\pm$0.002} & \textbf{0.304$\pm$0.003} & 0.628$\pm$0.003          & \textbf{0.319$\pm$0.001} \\ \bottomrule
        \end{tabular}
        }
\end{table*}

\begin{figure*}
    \centering
    \includegraphics[width=\linewidth]{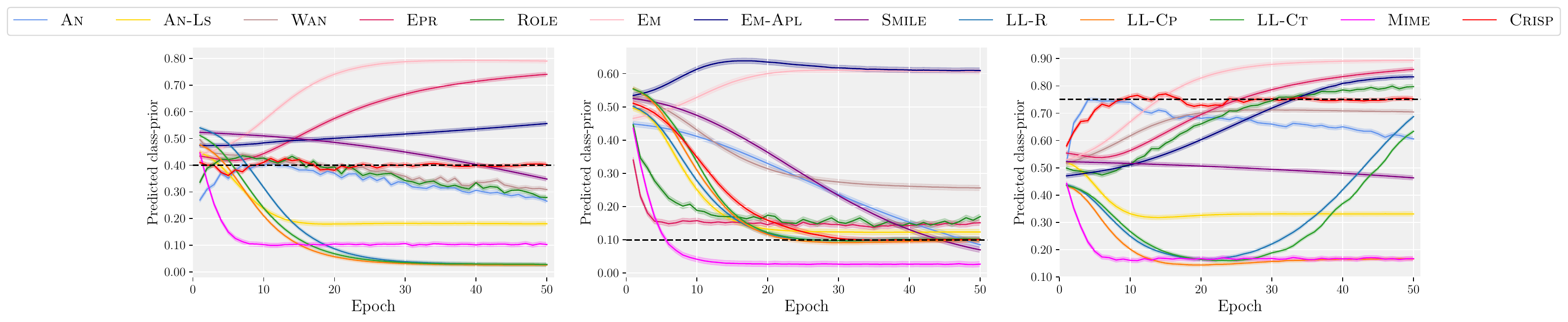}
    \caption{Predicted class-prior of \an \cite{cole2021multi}, \anls \cite{cole2021multi}, \wan \cite{cole2021multi}, \epr \cite{cole2021multi}, \role \cite{cole2021multi}, \emm \cite{cole2021multi}, \emapl \cite{zhou2022acknowledging}, \smilee \cite{xu2022one}, \smib \cite{liu2023revisiting}, \lln \cite{kim2022large} and {\proposed} on the \textit{$3$-rd} (left), \textit{$10$-th} (middle), and \textit{$12$-th} labels (right) of the dataset \texttt{Yeast}.}
    \label{fig:label_prior}
\end{figure*}

\begin{figure*}[t]
    \centering
    \begin{subfigure}{0.45\linewidth}
    \centering
    \includegraphics[width=\linewidth]{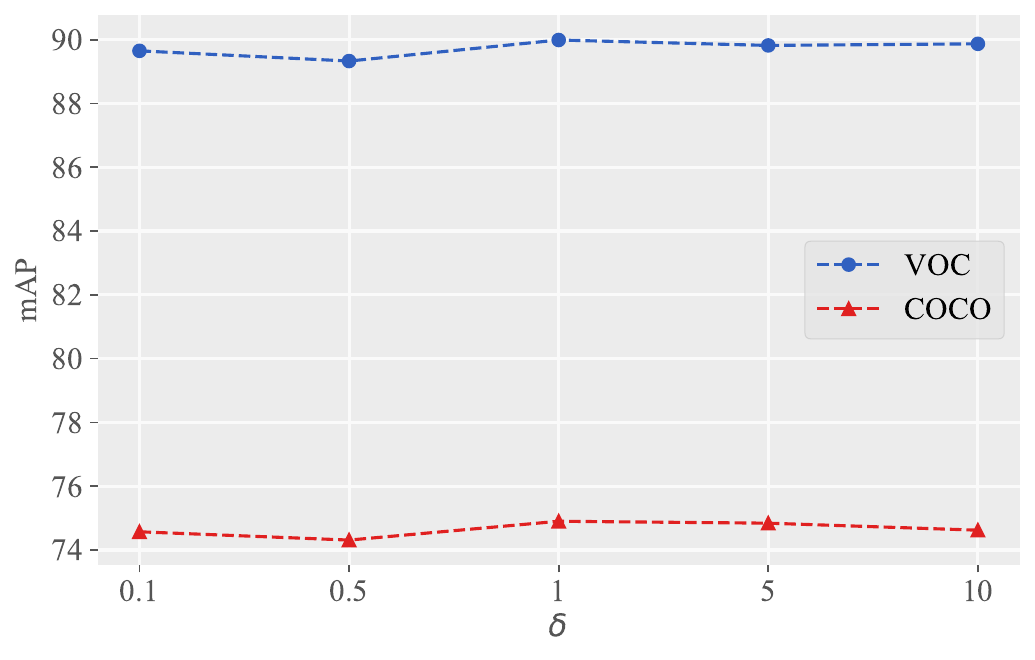}
    \caption{Sensitivity analysis of $\delta$}
    \label{delta}
    \end{subfigure}
    \begin{subfigure}{0.45\linewidth}
      \centering
      \includegraphics[width=\linewidth]{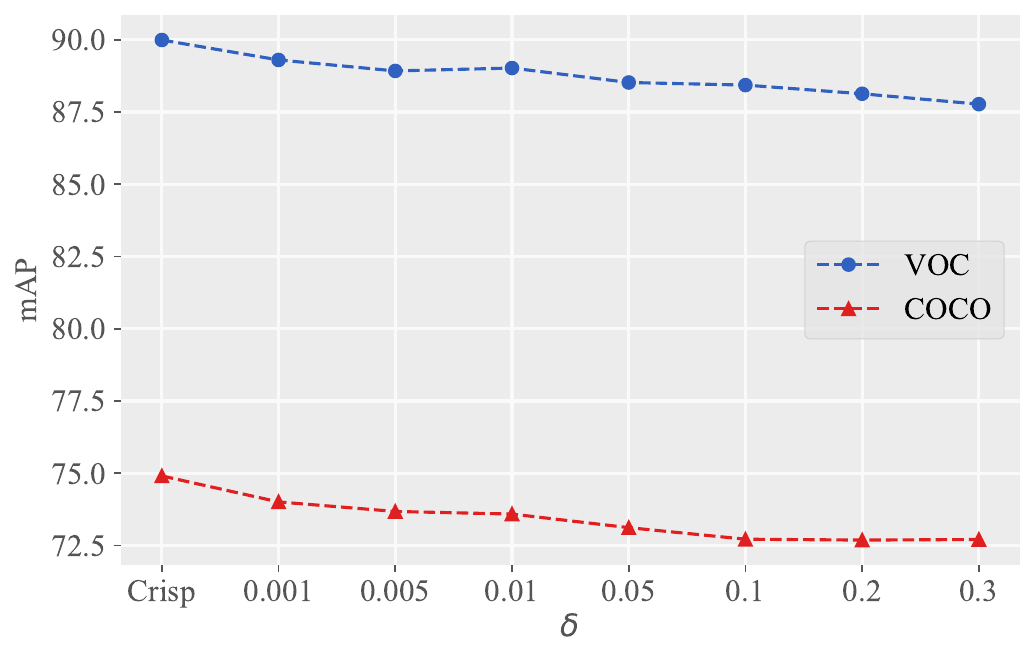}
      \caption{Sensitivity analysis of $\lambda$}
      \label{ablation}
  \end{subfigure}
    \caption{ (a) Parameter sensitivity analysis of $\delta$ ($ \tau $ is fixed as $ 0.01 $, $ \lambda $ is fixed as $ 1 $); 
    (b) The initial data point represents the performance of the proposed {\proposed} (with class-priors estimator). The others are the performance with a fixed value for all class-priors gradually increasing from $0.001$ to $0.3$.}
    \label{Convergence1}
\end{figure*}

\subsection{Experimental Results}
Table \ref{mAP} presents the comparison results of {\proposed} compared with other methods on \texttt{VOC},  \texttt{COCO}, \texttt{NUS}, and  \texttt{CUB}. The proposed method achieves the best performance on \texttt{VOC},  \texttt{COCO}, \texttt{NUS} and \texttt{CUB}.
Tables \ref{ranking} and \ref{ap} record the results of our method and other comparing methods on the MLL datasets in terms of \textit{Ranking loss} and \textit{Average precision} respectively. Similar results for other metrics can be found in Appendix \ref{MLL_results}. 
Note that due to the inability to compute the loss function of {\lagc} without data augmentation, we do not report the results of {\lagc} on MLL datasets because data augmentation techniques are not suitable for the MLL datasets.
Similarly, since the operations of {\boost} for CAM are not applicable to the tabular data in MLL datasets, its results are also not reported.
The results demonstrate that our proposed method consistently achieves desirable performance in almost all cases (except the result of \texttt{Mirflickr} on the metric \textit{Average Precision}, where our method attains a comparable performance against {\smilee}). Table \ref{Wilcoxon} in Appendix \ref{app:wilcoxon} reports the \textit{p}-values of the wilcoxon signed-ranks test \cite{demvsar2006statistical} for the corresponding tests and the statistical test results at 0.05 signiﬁcance level, which reveals that {\proposed} consistently outperforms other comparing algorithms (43 out of 60 test cases score \textbf{win}). These experimental results validate the effectiveness of {\proposed} in addressing SPMLL problems.



\subsection{Further Analysis}

\subsubsection{Class-Prior Prediction}
Figure \ref{fig:label_prior} illustrates the comparison results of the predicted class-priors of {\proposed} with other methods on the \textit{$3$-rd} (left), \textit{$10$-th} (middle), and \textit{$12$-th} labels (right) of the dataset \texttt{Yeast}. 
Compared with other approaches, whose predicted class-priors $p(\hat y_j = 1)$, which represents the expected value of the predicted, significantly deviate from the true class-priors, {\proposed} achieves consistent predicted class-priors with the ground-truth class-priors (black dashed lines). 
Without the constraint of the class-priors, the predicted class-prior probability diverges from the true class-prior as epochs increase, significantly impacting the model's performance.
In this experiment, the true class-priors are derived by calculating the statistical information for each dataset. More experimental results about the convergence analyses of estimated class-priors of all classes on MLIC datasets are recorded in Appendix \ref{MLIC_results}. These results demonstrate the necessity of incorporating class-priors in the training of the SPMLL model. 

\begin{table}[t!]
  \centering
  \caption{Time cost of class-priors estimation and the whole training time of one epoch.}
  \label{tab:time}
      \begin{tabular}{@{}lcccc@{}}
      \toprule
      & VOC & COCO & NUS & CUB \\ \midrule
      Class-priors estimation & 0.24 & 3.47 & 6.4 & 0.45 \\
      Whole training epoch & 2.19 & 27.29 & 49.09 & 3.89 \\
      \bottomrule
  \end{tabular}
  \footnotesize
\end{table}

\begin{table}[t!]
  \centering
  \caption{Predictive performance of {\proposed} with different updating frequency of class-priors estimation (3 epoch).}
  \label{tab:3ep}
      \begin{tabular}{@{}lcccc@{}}
      \toprule
      & VOC & COCO & NUS & CUB \\ \midrule
      \textsc{Crisp-3ep} & 89.077$\pm$0.251 & 73.930$\pm$0.399 & 49.463$\pm$0.216 & 19.450$\pm$0.389 \\
      {\proposed} & 89.931$\pm$0.014 & 74.913$\pm$0.235 & 50.045$\pm$0.112 & 22.150$\pm$0.028 \\
      \bottomrule
  \end{tabular}
  \footnotesize
\end{table}


\subsubsection{Sensitivity Analysis}
The performance sensitivity of the proposed {\proposed} approach with respect to its parameters $\delta$, $\tau$ and $ \lambda $ during the class-priors estimation phase is analyzed in this section. Figure \ref{delta} illustrates the performance of the proposed method on \texttt{VOC} and \texttt{COCO} under various parameter settings, where $\delta$ is incremented from $0.001$ to $0.1$. The parameter sensitivity analysis for $\tau$ and $\lambda$ are provided in Figure \ref{tau} and \ref{lambda} in Appendix \ref{app:sens}. 
The performance of the proposed method remains consistently stable across a wide range of parameter values. This characteristic is highly desirable as it allows for the robust application of the proposed method without the need for meticulous parameter fine-tuning, ensuring reliable classification results.

\subsubsection{Ablation Study}
Figure \ref{ablation} depicts the results of the ablation study to investigate the impact of the class-priors estimator by comparing it with a fixed value for all class-priors. 
The initial data point represents the performance of the proposed {\proposed} (with class-priors estimator). Subsequently, we maintain a fixed identical class-priors, gradually increasing it from $0.001$ to $0.3$. As expected, our method exhibits superior performance when utilizing the class-priors estimator, compared with employing a fixed class-prior proportion.
Furthermore, we conduct experiments comparing the performance of {\proposed} with the approach that estimating the class-priors with the full labels of validation set ({\textsc{Crisp-val}}). Table \ref{Crisp-val} shows that the performance of {\proposed} is superior to {\textsc{Crisp-val}}. 
It is indeed feasible to estimate the class-priors using the validation set. However, the size of validation set in many datasets is often quite small, which can lead to unstable estimation of the class-priors, thus leading to a suboptimal performance. Similar results are observed in Table \ref{val-mll} of Appendix \ref{app:val-mll} for the MLL datasets.

\begin{wraptable}{r}{0.5\linewidth}
  \centering
  \captionof{table}{Predictive performance comparing {\proposed} with the approach of estimating priors from the validation set (\textsc{Crisp-val}).}
  \label{Crisp-val}
    \begin{tabular}{@{}lcccccc@{}}
      \toprule
      & \textsc{Crisp-val} & {\proposed} \\ \midrule
      VOC & 89.585$\pm$0.318 & \textbf{89.931$\pm$0.014} \\
      COCO & 74.435$\pm$0.148 & \textbf{74.913$\pm$0.235} \\
      NUS & 49.230$\pm$0.113 & \textbf{50.045$\pm$0.112} \\
      CUB & 19.600$\pm$1.400 & \textbf{22.150$\pm$0.028} \\
      \bottomrule
    \end{tabular}
\end{wraptable}

\subsubsection{Time Cost of Class-Priors Estimation}\label{sec:time}
In Eq. (\ref{eq:threshold}), we have adopted an exhaustive search strategy to find an optimal threshold for estimating class-priors in each training epoch, which may introduce additional computational overhead to the algorithm. We conducted experimental analysis on this aspect. As illustrated in the Table \ref{tab:time}, the time for class-priors estimation is short compared to the overall training time for one epoch, ensuring that our method remains practical for use in larger datasets. Additionally, to further enhance the speed of our algorithm, we have experimented with updating the class-priors every few epochs instead of every single one in Table \ref{tab:3ep}. The variant of our method, denoted as \textsc{Crisp-3ep}, updates the priors every three epochs and our experiments show that this results in a negligible loss in performance.


\section{Conclusion}
In conclusion, this paper presents a novel approach to address the single-positive multi-label learning (SPMLL) problem by considering the impact of class-priors on the model. 
We propose a theoretically guaranteed class-priors estimation method that ensures the convergence of estimated class-prior to ground-truth class-priors during the training process. 
Furthermore, we introduce an unbiased risk estimator based on the estimated class-priors and derive a generalization error bound to guarantee that the obtained risk minimizer would approximately converge to the optimal risk minimizer of fully supervised learning.
Experimental results on ten MLL benchmark datasets demonstrate the effectiveness and superiority of our method over existing SPMLL approaches.




\bibliographystyle{plain}
\bibliography{neurips_2024.bbl}


\appendix
\newpage
\section{Appendix}


\subsection{Proof of Theorem \ref{thm:cpe}}\label{app:proof_cpe}

\begin{proof}
  Firstly, we have:
  \begin{equation}
    \begin{aligned}
      \Big\vert \frac{\hat q_j(z)}{\hat q_j^p(z)} - \frac{q_j(z)}{q_j^p(z)} \Big\vert &= \frac{\hat q_j(z)q_j^p(z) - \hat q_j^p(z)q_j(z)}{\hat q_j^p(z)q_j^p(z)} \\
      & \leq \frac{ \vert \hat q_j(z)q_j^p(z) - q_j^p(z)q_j(z) \vert + \vert q_j^p(z)q_j(z) - \hat q_j^p(z)q_j(z) \vert  }{\hat q_j^p(z)q_j^p(z)} \\
      & = \frac{1}{\hat q_j^p(z)}\vert \hat q_j(z) - q_j(z) \vert + \frac{q_j(z)}{\hat q_j^p(z)q_j^p(z)} \vert \hat q_j^p(z) - q_j^p(z) \vert,
    \end{aligned}
  \end{equation}
  where $ z $ is an arbitrary constant in $ [0,1] $. Using DKW inequality, we have with probability $ 1 - \delta $: $ \vert \hat q_j(z) - q_j(z) \vert \leq \sqrt{\frac{\log{2/\delta}}{2n}} $ and $ \vert \hat q_j^p(z) - q_j^p(z) \vert \leq \sqrt{\frac{\log{2/\delta}}{2n_j^p}} $. Therefore, with probability $ 1 - \delta $:
  \begin{equation}\label{eq:lemma1}
    \begin{aligned}
      \Big\vert \frac{\hat q_j(z)}{\hat q_j^p(z)} - \frac{q_j(z)}{q_j^p(z)} \Big\vert &\leq \frac{1}{\hat q_j^p(z)} \left( \sqrt{\frac{\log{4/\delta}}{2n}} + \frac{q_j(z)}{q_j^p(z)}\sqrt{\frac{\log{4/\delta}}{2n_j^p}} \right).
  \end{aligned}
\end{equation}
Then, we define:
\begin{equation*}
  \begin{aligned}
    &\hat{z} = \arg\min_{z\in [0,1]} \left( \frac{\hat q_j(z)}{\hat q_j^p(z)} + 
    \frac{1+\tau}{\hat q_j^p(z)}\left( \sqrt{\frac{\log(4/\delta)}{2n}} + \sqrt{\frac{\log(4/\delta)}{2n_j^p}} \right) \right), \\
    &z^\star = \arg\min_{z\in[0,1]}\frac{ q_j(z) }{ q_j^p(z) }, \\
    &\hat\pi_j = \frac{\hat q_j(\hat z)}{\hat q_j^p(\hat z)} \qquad\text{    and    }\qquad
    \pi_j^\star = \frac{ q_j(z^\star) }{ q_j^p(z^\star) }.
  \end{aligned}
\end{equation*}

Next, consider $ z^\prime\in [0,1] $ such that $ \hat q_j^p(z^\prime) = \frac{\tau}{2+\tau} \hat q_j^p(z^\star) $. We now show that $ \hat z < z^\prime $. For any $ z\in [0,1] $, by the DWK inequality, we have with probability $ 1-\delta $:
\begin{equation}
  \begin{aligned}
    \hat q_j^p(z) - \sqrt{\frac{\log{4/\delta}}{2n_j^p}}  \leq q_j^p(z),  \\
    q_j(z) - \sqrt{\frac{\log{4/\delta}}{2n}} \leq  \hat q_j(z).
  \end{aligned}
\end{equation}

Since $ \frac{ q_j(z^\star) }{ q_j^p(z^\star) } \leq \frac{ q_j(z) }{ q_j^p(z) } $, we have:
\begin{equation}
  \begin{aligned}
    \hat q_j(z) \geqslant q_j^p(z) \frac{q_j\left(z^\star\right)}{q_j^p\left(z^\star\right)}-\sqrt{\frac{\log (4 / \delta)}{2 n}} \geqslant\left(\hat{q}_j^p(z)-\sqrt{\frac{\log (4 / \delta)}{2 n_j^p}}\right) \frac{q_j\left(z^\star\right)}{q_j^p\left(z^\star\right)}-\sqrt{\frac{\log (4 / \delta)}{2 n}}.
  \end{aligned}
\end{equation}

Therefore, we have:
\begin{equation}
  \begin{aligned}
    \frac{\hat q_j(z)}{\hat q_j^p(z)} \geq \pi_j^\star - \frac{1}{\hat q_j^p(z)} \left( \sqrt{\frac{\log (4 / \delta)}{2 n}}  + \pi_j^\star \sqrt{\frac{\log (4 / \delta)}{2 n_j^p}} \right).
  \end{aligned}
\end{equation}

Using Eq. (\ref{eq:lemma1}) at $ z^\star $ and the fact that $ \pi_j^\star = \frac{ q_j(z^\star) }{ q_j^p(z^\star) }\leq 1 $ , we have:
\begin{equation}
  \begin{aligned}
    \frac{\hat q_j(z)}{\hat q_j^p(z)} \geq \frac{\hat q_j(z^\star)}{\hat q_j^p(z^\star)} - \left( \frac{1}{\hat q_j^p(z^\star)} + \frac{1}{\hat q_j^p(z)} \right) \left( \sqrt{\frac{\log (4 / \delta)}{2 n}}  + \pi_j^\star \sqrt{\frac{\log (4 / \delta)}{2 n_j^p}} \right) \\
    \geq \frac{\hat q_j(z^\star)}{\hat q_j^p(z^\star)}  - \left( \frac{1}{\hat q_j^p(z^\star)} + \frac{1}{\hat q_j^p(z)} \right) \left( \sqrt{\frac{\log (4 / \delta)}{2 n}} + \sqrt{\frac{\log (4 / \delta)}{2 n_j^p}} \right)  .
  \end{aligned}
\end{equation}
Furthermore, the upper confidence bound at $ z $ is lower bounded by:
\begin{equation}\label{eq:upper_bound}
  \begin{aligned}
    &\frac{\hat q_j(z)}{\hat q_j^p(z)} + 
    \frac{1+\tau}{\hat q_j^p(z)}\left( \sqrt{\frac{\log(4/\delta)}{2n}} + \sqrt{\frac{\log(4/\delta)}{2n_j^p}} \right) \\
    &\geq \frac{\hat q_j(z^\star)}{\hat q_j^p(z^\star)} + \left( \frac{1+\tau}{\hat q_j^p(z)} - \frac{1}{\hat q_j^p(z^\star)} - \frac{1}{\hat q_j^p(z)} \right)\left( \sqrt{\frac{\log(4/\delta)}{2n}} + \sqrt{\frac{\log(4/\delta)}{2n_j^p}} \right) \\
    &= \frac{\hat q_j(z^\star)}{\hat q_j^p(z^\star)} + \left( \frac{\tau}{\hat q_j^p(z)} - \frac{1}{\hat q_j^p(z^\star)} \right)\left( \sqrt{\frac{\log(4/\delta)}{2n}} + \sqrt{\frac{\log(4/\delta)}{2n_j^p}} \right). \\
  \end{aligned}
\end{equation}
Using Eq. (\ref{eq:upper_bound}) at $ z=z^\prime $ where $ \hat q_j^p(z^\prime) = \frac{\tau}{2+\tau} \hat q_j^p(z^\star) $, we have:
\begin{equation}
  \begin{aligned}
    &\frac{\hat q_j(z^\prime)}{\hat q_j^p(z^\prime)} + 
    \frac{1+\tau}{\hat q_j^p(z^\prime)}\left( \sqrt{\frac{\log(4/\delta)}{2n}} + \sqrt{\frac{\log(4/\delta)}{2n_j^p}} \right) \\
    &\geq \frac{\hat q_j(z^\star)}{\hat q_j^p(z^\star)} + \left( \frac{\tau}{\hat q_j^p(z^\prime)} - \frac{1}{\hat q_j^p(z^\star)} \right)\left( \sqrt{\frac{\log(4/\delta)}{2n}} + \sqrt{\frac{\log(4/\delta)}{2n_j^p}} \right). \\
    &\geq \frac{\hat q_j(z^\star)}{\hat q_j^p(z^\star)} + 
    \frac{1+\tau}{\hat q_j^p(z^\star)}\left( \sqrt{\frac{\log(4/\delta)}{2n}} + \sqrt{\frac{\log(4/\delta)}{2n_j^p}} \right).
  \end{aligned}
\end{equation}
Moreover from Eq. (\ref{eq:upper_bound}) and using definition of $ \hat z $, we have:
\begin{equation}
  \begin{aligned}
    &\frac{\hat q_j(z^\prime)}{\hat q_j^p(z^\prime)} + 
    \frac{1+\tau}{\hat q_j^p(z^\prime)}\left( \sqrt{\frac{\log(4/\delta)}{2n}} + \sqrt{\frac{\log(4/\delta)}{2n_j^p}} \right) \\
    &\geq \frac{\hat q_j(z^\star)}{\hat q_j^p(z^\star)} + 
    \frac{1+\tau}{\hat q_j^p(z^\star)}\left( \sqrt{\frac{\log(4/\delta)}{2n}} + \sqrt{\frac{\log(4/\delta)}{2n_j^p}} \right) \\
    &\geq \frac{\hat q_j(\hat z)}{\hat q_j^p(\hat z)} + 
    \frac{1+\tau}{\hat q_j^p(\hat z)}\left( \sqrt{\frac{\log(4/\delta)}{2n}} + \sqrt{\frac{\log(4/\delta)}{2n_j^p}} \right), \\
  \end{aligned}
\end{equation}
and hence $ \hat z \leq z^\prime $.

We now establish an upper and lower bound on $ \hat z $. By definition of $ \hat z $, we have:
\begin{equation}
  \begin{aligned}
    &\frac{\hat q_j(\hat z)}{\hat q_j^p(\hat z)} + 
    \frac{1+\tau}{\hat q_j^p(\hat z)}\left( \sqrt{\frac{\log(4/\delta)}{2n}} + \sqrt{\frac{\log(4/\delta)}{2n_j^p}} \right) \\
    &\leq \min_{z\in [0,1]} \left( \frac{\hat q_j(z)}{\hat q_j^p(z)} + 
    \frac{1+\tau}{\hat q_j^p(z)}\left( \sqrt{\frac{\log(4/\delta)}{2n}} + \sqrt{\frac{\log(4/\delta)}{2n_j^p}} \right) \right) \\
    &\leq \frac{\hat q_j(z^\star)}{\hat q_j^p(z^\star)} + 
    \frac{1+\tau}{\hat q_j^p(z^\star)}\left( \sqrt{\frac{\log(4/\delta)}{2n}} + \sqrt{\frac{\log(4/\delta)}{2n_j^p}} \right).
  \end{aligned}
\end{equation}
Using Eq. (\ref{eq:lemma1}) at $ z^\star $, we have:
\begin{equation}
  \begin{aligned}
    \frac{\hat q_j(z^\star)}{\hat q_j^p(z^\star)} \leq \frac{q_j(z^\star)}{ q_j^p(z^\star)} + 
    \frac{1}{\hat q_j^p(z^\star)}\left( \sqrt{\frac{\log(4/\delta)}{2n}} + \pi_j^\star \sqrt{\frac{\log(4/\delta)}{2n_j^p}} \right).
  \end{aligned}
\end{equation}
Then, we have:
\begin{equation}
  \begin{aligned}
    \hat\pi_j = \frac{\hat q_j(\hat z)}{\hat q_j^p(\hat z)} \leq \pi_j^\star + 
    \frac{2+\tau}{\hat q_j^p(z^\star)}\left( \sqrt{\frac{\log(4/\delta)}{2n}} + \sqrt{\frac{\log(4/\delta)}{2n_j^p}} \right).
  \end{aligned}
\end{equation}

Assume $ n_j^p \geq 2\frac{\log4/\delta}{{q_j^p}^2(z^\star)} $, we have $ \hat q_j^p(z^\star) \geq q_j^p(z^\star) / 2 $ and hence:
\begin{equation}\label{eq:23}
  \begin{aligned}
    \hat\pi_j \leq \pi_j^\star + \frac{4+2\tau}{q_j^p(z^\star)}\left( \sqrt{\frac{\log(4/\delta)}{2n}} + \sqrt{\frac{\log(4/\delta)}{2n_j^p}} \right).
  \end{aligned}
\end{equation}

From Eq. (\ref{eq:lemma1}) at $ \hat{z} $, we have:
\begin{equation}
  \begin{aligned}
    \frac{q_j(\hat z)}{q_j^p(\hat z)} \leq \frac{\hat q_j(\hat z)}{\hat q_j^p(\hat z)} + \frac{1}{\hat q_j^p(\hat z)}\left( \sqrt{\frac{\log(4/\delta)}{2n}} + \frac{q_j(\hat z)}{q_j^p(\hat z)}\sqrt{\frac{\log(4/\delta)}{2n_j^p}} \right).
  \end{aligned}
\end{equation}

Since $ \pi_j^\star \leq \frac{q_j(\hat z)}{q_j^p(\hat z)} $, we have:
\begin{equation}\label{eq:25}
  \begin{aligned}
    \pi_j^\star \leq \frac{q_j(\hat z)}{q_j^p(\hat z)} \leq \frac{\hat q_j(\hat z)}{\hat q_j^p(\hat z)} + \frac{1}{\hat q_j^p(\hat z)}\left( \sqrt{\frac{\log(4/\delta)}{2n}} + \frac{q_j(\hat z)}{q_j^p(\hat z)}\sqrt{\frac{\log(4/\delta)}{2n_j^p}} \right).
  \end{aligned}
\end{equation}

Using Eq. (\ref{eq:23}) and the assumption that $ n \geq n_j^p \geq 2\frac{\log4/\delta}{{q_j^p}^2(z^\star)} $ , we have:
\begin{equation}
  \begin{aligned}
    &\hat\pi_j = \frac{\hat q_j(\hat z)}{\hat q_j^p(\hat z)} \leq \pi_j^\star + 
    \frac{4+2\tau}{q_j^p(z^\star)}\left( \sqrt{\frac{\log(4/\delta)}{2n}} + \sqrt{\frac{\log(4/\delta)}{2n_j^p}} \right) \\
    &\leq \pi_j^\star + {4+2\tau} \leq 1 + {4+2\tau} = 5+2\tau.
  \end{aligned}
\end{equation}
Using this in Eq. (\ref{eq:25}), we have:
\begin{equation}
  \begin{aligned}
    \pi_j^\star \leq \frac{\hat q_j(\hat z)}{\hat q_j^p(\hat z)} + \frac{1}{\hat q_j^p(\hat z)}\left( \sqrt{\frac{\log(4/\delta)}{2n}} + (5+2\tau)\sqrt{\frac{\log(4/\delta)}{2n_j^p}} \right).
  \end{aligned}
\end{equation}
Since $ \hat z \leq z^\prime $, we have $ \hat q_j^p(\hat z) \geq \hat q_j^p(z^\prime) = \frac{\tau}{2+\tau} \hat q_j^p(z^\star) $. Therefore, we have:
\begin{equation}
  \begin{aligned}
    \pi_j^\star - \frac{2+\tau}{\tau\hat q_j^p(z^\star)}\left( \sqrt{\frac{\log(4/\delta)}{2n}} + (5+2\tau)\sqrt{\frac{\log(4/\delta)}{2n_j^p}} \right) \leq \frac{\hat q_j(\hat z)}{\hat q_j^p(\hat z)} = \hat\pi_j.
  \end{aligned}
\end{equation}
With the assumption that $ n_j^p \geq 2\frac{\log4/\delta}{{q_j^p}^2(z^\star)} $, we have $ \hat q_j^p(z^\star) \geq q_j^p(z^\star) / 2 $, which implies:
\begin{equation}
  \begin{aligned}
    \pi_j^\star - \frac{4+2\tau}{\tau q_j^p(z^\star)}\left( \sqrt{\frac{\log(4/\delta)}{2n}} + (5+2\tau)\sqrt{\frac{\log(4/\delta)}{2n_j^p}} \right) \leq \hat\pi_j.
  \end{aligned}
\end{equation}

Note that since $ \pi_j \leq \pi_j^\star $, the lower bound remains the same as in Theorem \ref{thm:cpe}. For the upper bound, with $ q_j(z^\star)=\pi_jq_j^p(z^\star)+(1-\pi_j)q_j^n(z^\star) $, we have $ \pi_j^\star = \pi_j + (1-\pi_j)\frac{q_j^n(z^\star)}{q_j^p(z^\star)} $. Then the proof is completed.

\end{proof}

\subsection{The details of the optimization of Eq. (\ref{eq:threshold})}\label{app:threshold}
In practice, to determine the optimal threshold, we conduct an exhaustive search across the set of outputs generated by the function $ f^j $ for each class. For instance, for a given class $ j $, and a set of instances $ \boldsymbol{x}_1, \boldsymbol{x}_2, \boldsymbol{x}_3 $ in our dataset, we compute the corresponding outputs $ z_1 = f^j(\boldsymbol{x}_1), z_2 = f^j(\boldsymbol{x}_2), z_3 = f^j(\boldsymbol{x}_3) $.

The optimal threshold $ \hat{z} $ is then selected by identifying the value of $ z\in\{z_1, z_2, z_3\} $ that minimizes the objective function specified in Equation (2):

$$
\hat{z} = \arg\min_{z\in \{z_1, z_2, z_3\}} \left( \frac{\hat{q}_j(z)}{\hat{q}_j^p(z)} + \frac{1+\tau}{\hat{q}_j^p(z)}\left( \sqrt{\frac{\log(4/\delta)}{2n}} + \sqrt{\frac{\log(4/\delta)}{2n_j^p}} \right) \right)
$$
This approach ensures that we find the optimal threshold that minimizes the given expression, as per Eq. (\ref{eq:threshold}), across all available output values from the function $f^j$.

\subsection{Details of Eq. (\ref{eq:full_label_loss_func})}\label{app:proof_risk}
\begin{equation}
	\begin{aligned}
		\mathcal{R}(f) &= \mathbb{E}_{(\bm x, \bm y)\sim p(\bm x, \bm y)} \left[ \mathcal{L}(f(\bm x), \bm y) \right] \\
		& = \int_{\bm x} \sum_{\bm y} \mathcal{L}(f(\bm x), \bm y) p(\bm x \vert \bm y) p(\bm y) d \bm x \\
		& = \sum_{\bm y} p(\bm y) \int_{\bm x} \mathcal{L}(f(\bm x), \bm y) p(\bm x \vert \bm y) d \bm x \\
		& = \sum_{\bm y} p(\bm y) \mathbb{E}_{\bm x \sim p(\bm x \vert \bm y)} \left[ \mathcal{L}(f(\bm x), \bm y) \right]. \\
	\end{aligned}
\end{equation}

\subsection{Details of Eq. (\ref{eq:full_label_loss_func2})}\label{app:proof_risk2}
In fact, any symmetric loss function can work, here we adopt the absolute loss function. The absolute loss function is $ \ell(f^j(\bm x), y_j) = |f^j(\bm x)-y_j| $, when $ y_j=1 $, $ \ell(f^j(\bm x), 1)=|1-f^j(\bm x)| $, and when $ y_j=0 $, $ \ell(f^j(\bm x), 0)=f^j(\bm x) $.  Then:
\begin{equation}
  \begin{aligned}
    \mathcal{R}(f) & = \sum_{\bm y} p(\bm y) \mathbb{E}_{\bm x \sim p(\bm x \vert \bm y)} \left[ \sum_{j=1}^c y_j\ell(f^j(\bm x), 1) + (1 - y_j)\ell(f^j(\bm x), 0)  \right] \\
		& = \sum_{j=1}^c p(y_j = 1) \mathbb{E}_{\bm x \sim p(\bm x \vert y_j = 1)}\left[ \ell(f^j(\bm x), 1) \right] 
    + p(y_j = 0) \mathbb{E}_{\bm x \sim p(\bm x \vert y_j = 0)}\left[ \ell(f^j(\bm x), 0) \right] \\
		& = \sum_{j=1}^c p(y_j = 1) \mathbb{E}_{\bm x \sim p(\bm x \vert y_j = 1)}\left[ 1 - f^j(\bm x) \right] 
    + (1 - p(y_j = 1)) \mathbb{E}_{\bm x \sim p(\bm x \vert y_j = 0)}\left[ f^j(\bm x) \right] \\
		& = \sum_{j=1}^c p(y_j = 1) \mathbb{E}_{\bm x \sim p(\bm x \vert y_j = 1)}\left[ 1 - f^j(\bm x) \right] + \mathbb{E}_{\bm x \sim p(\bm x)}\left[f^j(\bm x)\right] \\
    & \qquad - p(y_j = 1) \mathbb{E}_{\bm x \sim p(\bm x \vert y_j = 1)}\left[f^j(\bm x) \right] \\
		& = \sum_{j=1}^c p(y_j = 1) \mathbb{E}_{\bm x \sim p(\bm x \vert y_j = 1)}\left[ 1 - f^j(\bm x) \right] + \mathbb{E}_{\bm x \sim p(\bm x)}\left[f^j(\bm x)\right] \\
    & \qquad - p(y_j = 1) \mathbb{E}_{\bm x \sim p(\bm x \vert y_j = 1)}\left[f^j(\bm x) - 1 + 1 \right] \\
		& = \sum_{j=1}^c 2p(y_j = 1) \mathbb{E}_{\bm x \sim p(\bm x \vert y_j = 1)}\left[ 1 - f^j(\bm x) \right] + \mathbb{E}_{\bm x \sim p(\bm x)}\left[f^j(\bm x)\right] - p(y_j = 1).
  \end{aligned}
\end{equation}

\subsection{Proof of Theorem \ref{theorem:theorem1}}\label{app:proof_error_bound}

In this subsection, an estimation error bound is established for Eq. (\ref{eq:empirical_risk_estimator}) to demonstrate its learning consistency. 
Specifically, The derivation of the estimation error bound involves two main parts, each corresponding to one of the loss terms in Eq. (\ref{eq:empirical_risk_estimator}). 
The empirical risk estimator according to Eq. (\ref{eq:empirical_risk_estimator}) can be written as:
\begin{equation}
  \begin{aligned}
    \widehat{\mathcal{R}}_{sp}(f) &= \sum_{j=1}^c \frac{2\pi_j}{\vert \mathcal{S}_{L_j} \vert} \sum_{\bm x \in \mathcal{S}_{L_j}} \left(1 - f^j(\bm x)\right) 
    + \frac{1}{n} \sum_{\bm x \in \tilde{\mathcal{D}}}\left(f^j(\bm x) - \pi_j\right) \\
    &= \widehat{\mathcal{R}}_{sp}^L(f)+\widehat{\mathcal{R}}_{sp}^U(f),
  \end{aligned}
\end{equation}
Firstly, we define the function spaces as:
\begin{equation*}
  \begin{aligned}
    \mathcal{G}_{sp}^L &= \Big\{(\bm x, \bm l)\mapsto\sum_{j=1}^c 2\pi_j l_j \left(1 - f^j(\bm x)\right) \vert f \in \mathcal{F} \Big\},
    \mathcal{G}_{sp}^U = \Big\{(\bm x, \bm l)\mapsto\sum_{j=1}^c \left(f^j(\bm x) - \pi_j \right) \vert f \in \mathcal{F} \Big\}, \\
  \end{aligned}
\end{equation*}
and denote the expected Rademacher complexity \citep{Foundations} of the function spaces as:
\begin{equation*}
  \begin{aligned}
    \widetilde{\mathfrak{R}}_n\left(\mathcal{G}_{sp}^L\right)&=\mathbb{E}_{\bm{x}, \bm l, \bm{\sigma}}\left[\sup _{g \in \mathcal{G}_{sp}^L} \sum_{i=1}^n \sigma_i g\left(\bm{x}_i, \bm l_i\right)\right],\\
    \widetilde{\mathfrak{R}}_n\left(\mathcal{G}_{sp}^U \right)&=\mathbb{E}_{\bm{x}, \bm l, \bm{\sigma}}\left[\sup _{g \in \mathcal{G}_{sp}^U} \sum_{i=1}^n \sigma_i g\left(\bm{x}_i, \bm l_i\right)\right],
  \end{aligned}
\end{equation*}
where $ \bm\sigma=\left\{ \sigma_1, \sigma_2, \cdots, \sigma_n \right\} $ is $ n $ Rademacher variables with $ \sigma_i $ independently uniform variable taking value in $ \{+1,-1\} $. Then we have:
\begin{lemma}\label{lemma:lemma1}
  We suppose that the loss function $ \mathcal{L}_{sp}^L = \sum_{j=1}^c 2\pi_j l_j \left(1 - f^j(\bm x)\right) $ and $ \mathcal{L}_{sp}^U = \sum_{j=1}^c \left(f^j(\bm x) - \pi_j \right) $  could be bounded by $ M $, i.e., $ M=\sup_{\bm x\in \mathcal{X}, f\in\mathcal{F}, \bm l\in \mathcal{Y}} \max(\mathcal{L}_{sp}^L(f(\bm x), \bm l), \mathcal{L}_{sp}^U(f(\bm x), \bm l)) $, and for any $ \delta > 0 $, with probability at least $ 1-\delta $, we have:
  \begin{equation*}
    \begin{aligned}
      &\sup_{f\in\mathcal{F}} \vert \mathcal{R}_{sp}^L(f) - \hat{\mathcal{R}}_{sp}^L (f) \vert \leq \frac{2}{C}	\widetilde{\mathfrak{R}}_n\left(\mathcal{G}_{sp}^L\right) + \frac{M}{2 \min_{j}\vert \mathcal{S}_{L_j} \vert} \sqrt{\frac{\log\frac{2}{\delta}}{2n}}, \\
      &\sup_{f\in\mathcal{F}} \vert \mathcal{R}_{sp}^U(f) - \hat{\mathcal{R}}_{sp}^U (f) \vert \leq 2	\widetilde{\mathfrak{R}}_n\left(\mathcal{G}_{sp}^U\right) + \frac{M}{2} \sqrt{\frac{\log\frac{2}{\delta}}{2n}},
    \end{aligned}
  \end{equation*}
\end{lemma}
where $ \mathcal{R}_{sp}^L(f) = \sum_{j=1}^c 2\pi_j \mathbb{E}_{\bm x \sim p(\bm x \vert y_j = 1)}\left[ 1 - f^j(\bm x) \right] $, $ \mathcal{R}_{sp}^U(f) = \mathbb{E}_{\bm x \sim p(\bm x)} \sum_{j=1}^{c}\left[f^j(\bm x)\right] - \pi_j $ and $ C = \min_{j} \mathbb{E}_{\tilde{\mathcal{D}}} \left[ \sum_{i=1}^{n} l_i^j \right] $ is a constant.

\begin{proof}
  Suppose an example $ (\bm x, \bm l) $ is replaced by another arbitrary example $ (\bm x', \bm l') $, then the change of $ \sup_{f\in\mathcal{F}} \mathcal{R}_{sp}^L(f) - \hat{\mathcal{R}}_{sp}^L (f) $ is no greater than $ \frac{M}{2n \min_{j}\vert \mathcal{S}_{L_j} \vert} $. 
  By applying McDiarmid's inequality, for any $ \delta > 0 $, with probility at least $ 1 - \frac{\delta}{2} $, 
  \begin{equation*}
    \begin{aligned}
      \sup_{f\in\mathcal{F}} \mathcal{R}_{sp}^L(f) - \hat{\mathcal{R}}_{sp}^L (f) \leq \mathbb{E}\left[ \sup_{f\in\mathcal{F}} \mathcal{R}_{sp}^L(f) - \hat{\mathcal{R}}_{sp}^L (f) \right] + \frac{M}{2 \min_{j}\vert \mathcal S_{L_j} \vert} \sqrt{\frac{\log\frac{2}{\delta}}{2n}}.
    \end{aligned}
  \end{equation*}
  By symmetry, we can obtain
  \begin{equation*}
    \begin{aligned}
      \sup_{f\in\mathcal{F}} \vert \mathcal{R}_{sp}^L(f) - \hat{\mathcal{R}}_{sp}^L (f) \vert \leq \mathbb{E}\left[ \sup_{f\in\mathcal{F}} \mathcal{R}_{sp}^L(f) - \hat{\mathcal{R}}_{sp}^L (f) \right] + \frac{M}{2 \min_{j}\vert \mathcal{S}_{L_j} \vert} \sqrt{\frac{\log\frac{2}{\delta}}{2n}}.
    \end{aligned}
  \end{equation*}

  Next is to bound the term $ \mathbb{E}\left[ \sup_{f\in\mathcal{F}} \mathcal{R}_{sp}^L(f) - \hat{\mathcal{R}}_{sp}^L (f) \right] $:
\begin{equation*}
	\begin{aligned}
		&\mathbb{E}\left[ \sup_{f\in\mathcal{F}} \mathcal{R}_{sp}^L(f) - \hat{\mathcal{R}}_{sp}^L (f) \right] = \mathbb{E}_{\tilde{\mathcal{D}}}\left[ \sup_{f\in\mathcal{F}} \mathcal{R}_{sp}^L(f) - \hat{\mathcal{R}}_{sp}^L (f) \right] \\
		&= \mathbb{E}_{\tilde{\mathcal{D}}}\left[ \sup_{f\in\mathcal{F}} \mathbb{E}_{\tilde{\mathcal
		D}^\prime} \left[  \hat{\mathcal{R}}_{sp}^{\prime L} (f) - \hat{\mathcal{R}}_{sp}^L (f) \right] \right] \\
		&\leq \mathbb{E}_{\tilde{\mathcal{D}}, \tilde{\mathcal{D}^\prime}}\left[ \sup_{f\in\mathcal{F}} \left[  \hat{\mathcal{R}}_{sp}^{\prime L}(f) - \hat{\mathcal{R}}_{sp}^L (f) \right] \right] \\
		&= \mathbb{E}_{\tilde{\mathcal{D}}, \tilde{\mathcal{D}^\prime}, \bm\sigma}\left[ \sup_{f\in\mathcal{F}} \sum_{i=1}^n\sum_{j=1}^c 
		\sigma_i \left( \frac{2\pi_j}{\sum_{i=1}^n {l^\prime}_i^j} {l^\prime}_i^j \left(1 - f^j(\bm x_i^\prime)\right)
		-  \frac{2\pi_j}{\sum_{i=1}^n {l}_i^j} {l}_i^j \left(1 - f^j(\bm x_i)\right) \right)
		 \right] \\
		&\leq \mathbb{E}_{\tilde{\mathcal{D}^\prime}, \bm\sigma}\left[ \sup_{f\in\mathcal{F}} \sum_{i=1}^n\sum_{j=1}^c 
		\sigma_i \left( \frac{2\pi_j}{\sum_{i=1}^n {l^\prime}_i^j} {l^\prime}_i^j \left(1 - f^j(\bm x_i^\prime)\right) \right) \right] \\
		& + \mathbb{E}_{\tilde{\mathcal{D}}, \bm\sigma}\left[ \sup_{f\in\mathcal{F}} \sum_{i=1}^n\sum_{j=1}^c 
		\sigma_i \left( \frac{2\pi_j}{\sum_{i=1}^n {l}_i^j} {l}_i^j \left(1 - f^j(\bm x_i)\right) \right) \right] \\
		& \leq \frac{1}{C} \mathbb{E}_{\tilde{\mathcal{D}^\prime}, \bm\sigma}\left[ \sup_{f\in\mathcal{F}} \sum_{i=1}^n\sum_{j=1}^c 
		\sigma_i \left( 2\pi_j {l^\prime}_i^j \left(1 - f^j(\bm x_i^\prime)\right) \right) \right] \\
		& + \frac{1}{C} \mathbb{E}_{\tilde{\mathcal{D}}, \bm\sigma}\left[ \sup_{f\in\mathcal{F}} \sum_{i=1}^n\sum_{j=1}^c 
		\sigma_i \left( 2\pi_j {l}_i^j \left(1 - f^j(\bm x_i)\right) \right) \right] \\
		& = \frac{2}{C}	\widetilde{\mathfrak{R}}_n\left(\mathcal{G}_{sp}^L\right),
	\end{aligned}
\end{equation*}
where $ C $ is a constant that $ C = \min_{j} \mathbb{E}_{\tilde{\mathcal{D}}} \left[ \sum_{i=1}^{n} y_i^j \right] $. Then we have:
\begin{equation*}
  \begin{aligned}
    \sup_{f\in\mathcal{F}} \vert \mathcal{R}_{sp}^L(f) - \hat{\mathcal{R}}_{sp}^L (f) \vert \leq \frac{2}{C}	\widetilde{\mathfrak{R}}_n\left(\mathcal{G}_{sp}^L\right) + \frac{M}{2 \min_{j}\vert \mathcal{S}_{L_j} \vert} \sqrt{\frac{\log\frac{2}{\delta}}{2n}}. \\
  \end{aligned}
\end{equation*}

Similarly, we can obtain:
\begin{equation*}
  \begin{aligned}
    \sup_{f\in\mathcal{F}} \vert \mathcal{R}_{sp}^U(f) - \hat{\mathcal{R}}_{sp}^U (f) \vert \leq 2	\widetilde{\mathfrak{R}}_n\left(\mathcal{G}_{sp}^U\right) + \frac{M}{2} \sqrt{\frac{\log\frac{2}{\delta}}{2n}},
  \end{aligned}
\end{equation*}

\end{proof}

\begin{lemma}\label{lemma:lemma2}
  Define $ \rho = \max_j 2\pi_j $, $ \mathcal{H}_j = \left\{ h: \bm x \mapsto f^j(\bm x) \vert f\in \mathcal{F} \right\} $ and $ \mathfrak{R}_n\left(\mathcal{H}_j\right)=\mathbb{E}_{p(\bm{x})} \mathbb{E}_{\bm{\sigma}}\left[\sup _{h \in \mathcal{H}_j} \frac{1}{n} \sum_{i=1}^n h\left(\bm{x}_i\right)\right] $. 
  Then, we have with Rademacher vector contraction inequality:
  \begin{equation*}
    \begin{aligned}
      \widetilde{\mathfrak{R}}_n\left(\mathcal{G}_{sp}^L\right) \leq \sqrt{2}\rho \sum_{j=1}^{c}\mathfrak{R}_n(\mathcal{H}_j), \qquad
      \widetilde{\mathfrak{R}}_n\left(\mathcal{G}_{sp}^U\right) \leq \sqrt{2} \sum_{j=1}^{c}\mathfrak{R}_n(\mathcal{H}_j),
    \end{aligned}
  \end{equation*}
\end{lemma}

Based on Lemma \ref{lemma:lemma1} and Lemma \ref{lemma:lemma2}, we could obtain the following theorem.
\begin{theorem}
  Assume the loss function $ \mathcal{L}_{sp}^L = \sum_{j=1}^c 2\pi_j l_j \left(1 - f^j(\bm x)\right) $ and $ \mathcal{L}_{sp}^U = \sum_{j=1}^c \left(f^j(\bm x) - \pi_j \right) $  could be bounded by $ M $, i.e., $ M=\sup_{\bm x\in \mathcal{X}, f\in\mathcal{F}, \bm l\in \mathcal{Y}} \max(\mathcal{L}_{sp}^L(f(\bm x), \bm l), \mathcal{L}_{sp}^U(f(\bm x), \bm y)) $, with probability at least $ 1-\delta $, we have:
  \begin{equation*}
    \begin{aligned}
      \mathcal{R}(\hat{f}_{sp}) - \mathcal{R}(f^\star) & \leq \frac{4}{C} \sum_{j=1}^{c}\widetilde{\mathfrak{R}}_n\left(\mathcal{G}_{sp}^L\right) + \frac{M}{\min_{j}\vert \mathcal S_{L_j} \vert} \sqrt{\frac{\log\frac{4}{\delta}}{2n}}
      +  4 \widetilde{\mathfrak{R}}_n\left(\mathcal{G}_{sp}^U\right) + M \sqrt{\frac{\log\frac{4}{\delta}}{2n}} \\
      &\leq \frac{4\sqrt{2}\rho}{C} \sum_{j=1}^{c}\mathfrak{R}_n(\mathcal{H}_j) + \frac{M}{\min_{j}\vert \mathcal S_{L_j} \vert} \sqrt{\frac{\log\frac{4}{\delta}}{2n}} 
      + 4\sqrt{2} \sum_{j=1}^{c}\mathfrak{R}_n(\mathcal{H}_j) + M \sqrt{\frac{\log\frac{4}{\delta}}{2n}}.
    \end{aligned}
  \end{equation*}
\end{theorem}

\begin{proof}
  \begin{equation*}
    \begin{aligned}
      \mathcal{R}(\hat{f}_{sp}) - \mathcal{R}(f^\star) &= \mathcal{R}(\hat{f}_{sp}) - \hat{\mathcal{R}}_{sp}(\hat{f}) + \hat{\mathcal{R}}_{sp}(\hat{f}) - \hat{\mathcal{R}}_{sp}(f^\star) + \hat{\mathcal{R}}_{sp}(f^\star) - \mathcal{R}(f^\star) \\
      & \leq \mathcal{R}(\hat{f}_{sp}) - \hat{\mathcal{R}}_{sp}(\hat{f}) + \hat{\mathcal{R}}_{sp}(f^\star) - \mathcal{R}(f^\star) \\
      & = \mathcal{R}_{sp}^L (\hat{f}_{sp}) - \hat{\mathcal{R}}_{sp}^L(\hat{f}) + \hat{\mathcal{R}}_{sp}^L (f^\star) - \mathcal{R}_{sp}^L (f^\star) \\
      & + \mathcal{R}_{sp}^U (\hat{f}_{sp}) - \hat{\mathcal{R}}_{sp}^U(\hat{f}) + \hat{\mathcal{R}}_{sp}^U (f^\star) - \mathcal{R}_{sp}^U (f^\star) \\
      & \leq 2 \sup_{f\in\mathcal{F}} \vert \mathcal{R}_{sp}^L(f) - \hat{\mathcal{R}}_{sp}^L (f) \vert + 2 \sup_{f\in\mathcal{F}} \vert \mathcal{R}_{sp}^U(f) - \hat{\mathcal{R}}_{sp}^U (f) \vert \\
      & \leq \frac{4}{C} \widetilde{\mathfrak{R}}_n\left(\mathcal{G}_{sp}^L\right) + \frac{M}{\min_{j}\vert \mathcal S_{L_j} \vert} \sqrt{\frac{\log\frac{4}{\delta}}{2n}}
      +  4 \widetilde{\mathfrak{R}}_n\left(\mathcal{G}_{sp}^U\right) + M \sqrt{\frac{\log\frac{4}{\delta}}{2n}} \\
      &\leq \frac{4\sqrt{2}\rho}{C} \sum_{j=1}^{c}\mathfrak{R}_n(\mathcal{H}_j) + \frac{M}{\min_{j}\vert \mathcal S_{L_j} \vert} \sqrt{\frac{\log\frac{4}{\delta}}{2n}} 
      + 4\sqrt{2} \sum_{j=1}^{c}\mathfrak{R}_n(\mathcal{H}_j) + M \sqrt{\frac{\log\frac{4}{\delta}}{2n}}.
    \end{aligned}
  \end{equation*}
\end{proof}

\begin{table}[t]
    \centering
    \caption{Characteristics of the MLIC datasets.}
    \label{image_ds}
    \begin{tabular}{ccccc}
    \toprule
    Dataset & \#Training & \#Validation & \#Testing & \#Classes \\ \midrule
    VOC     & 4574       & 1143         & 5823      & 20       \\
    COCO    & 65665      & 16416        & 40137     & 80       \\
    NUS     & 120000     & 30000        & 60260     & 81       \\
    CUB     & 4795       & 1199         & 5794      & 312      \\ \bottomrule
    \end{tabular}
\end{table}

\begin{table}[t]
    \centering
    \caption{Characteristics of the MLL datasets.}
    \label{mll_ds}
    \begin{tabular}{ccccc}
    \toprule
    Dataset     & \#Examples & \#Features & \#Classes & \#Domain \\ \midrule
    Image       & 2000       & 294        & 5       & Images   \\
    Scene       & 2407       & 294        & 6       & Images   \\
    Yeast       & 2417       & 103        & 14      & Biology  \\
    Corel5k     & 5000       & 499        & 374     & Images  \\
    Mirflickr   & 24581      & 1000       & 38      & Images  \\
    Delicious   & 16091      & 500        & 983     & Text \\ \bottomrule
    \end{tabular}
\end{table}

\begin{table*}[t]
    \caption{Predictive performance of each comparing method on MLL datasets in terms of \textit{Coverage} (mean $ \pm $ std). The best performance is highlighted in bold (the smaller the better).}
        \label{coverage}
        \centering
        \resizebox{0.8\linewidth}{!}
        {
        \begin{tabular}{ccccccc}
        \toprule
                     & Image       & Scene       & Yeast       & Corel5k    & Mirflickr   & Delicious           \\ \midrule
        {\an}           & 0.374$\pm$0.050          & 0.279$\pm$0.094          & 0.707$\pm$0.045          & 0.330$\pm$0.001          & 0.342$\pm$0.003          & 0.653$\pm$0.001          \\
        {\anls}          & 0.334$\pm$0.033          & 0.217$\pm$0.052          & 0.703$\pm$0.012          & 0.441$\pm$0.009          & 0.433$\pm$0.015          & 0.830$\pm$0.016          \\
        {\wan}           & 0.313$\pm$0.040          & 0.192$\pm$0.019          & 0.512$\pm$0.045          & 0.309$\pm$0.001          & 0.334$\pm$0.002          & 0.632$\pm$0.001          \\
        {\epr}          & 0.352$\pm$0.043          & 0.254$\pm$0.046          & 0.506$\pm$0.011          & 0.328$\pm$0.001          & 0.332$\pm$0.002          & 0.637$\pm$0.001          \\
        {\role}          & 0.306$\pm$0.049          & 0.157$\pm$0.023          & 0.519$\pm$0.026          & 0.551$\pm$0.007          & 0.448$\pm$0.028          & 0.887$\pm$0.004          \\
        {\emm}            & 0.407$\pm$0.036          & 0.281$\pm$0.096          & 0.575$\pm$0.042          & 0.382$\pm$0.005          & 0.359$\pm$0.010          & 0.753$\pm$0.004          \\
        {\emapl}       & 0.438$\pm$0.022          & 0.360$\pm$0.057          & 0.556$\pm$0.045          & 0.335$\pm$0.005          & 0.369$\pm$0.005          & 0.765$\pm$0.006          \\
        {\smilee}         & 0.242$\pm$0.014          & 0.146$\pm$0.037          & 0.462$\pm$0.003          & 0.308$\pm$0.007          & 0.328$\pm$0.004          & 0.628$\pm$0.003          \\
        {\smib}           & 0.265$\pm$0.002             & 0.183$\pm$0.002             & 0.476$\pm$0.004               & 0.317$\pm$0.008               & 0.329$\pm$0.331           & 0.689$\pm$0.003    \\ \midrule
        {\llr}  & 0.311$\pm$0.059 & 0.141$\pm$0.017 & 0.512$\pm$0.002 & 0.274$\pm$0.002          & 0.335$\pm$0.006 & 0.622$\pm$0.001 \\
        {\llcp} & 0.296$\pm$0.031 & 0.136$\pm$0.016 & 0.518$\pm$0.001 & \textbf{0.272$\pm$0.008} & 0.337$\pm$0.005 & 0.708$\pm$0.004          \\
        {\llct} & 0.297$\pm$0.017 & 0.161$\pm$0.031 & 0.509$\pm$0.001 & 0.277$\pm$0.005          & 0.335$\pm$0.003 & 0.708$\pm$0.002       \\ \midrule
        {\proposed} & \textbf{0.164$\pm$0.012} & \textbf{0.082$\pm$0.018} & \textbf{0.455$\pm$0.002} & 0.276$\pm$0.002 & \textbf{0.324$\pm$0.001} & \textbf{0.620$\pm$0.001} \\ \bottomrule
        \end{tabular}
        }
\end{table*}

\begin{table*}[t]
  \centering
  \caption{Predictive performance of each comparing methods on MLL datasets in terms of \textit{Hamming loss} (mean $\pm$ std). The best performance is highlighted in bold (the smaller the better).}
  
  \label{hamming}
  \resizebox{0.8\columnwidth}{!}{%
  \begin{tabular}{@{}ccccccc@{}}
  \toprule
                    & Image       & Scene       & Yeast       & Corel5k    & Mirflickr   & Delicious    \\ \midrule
  {\an}      & 0.229$\pm$0.000    & 0.176$\pm$0.001    & 0.306$\pm$0.000          & \textbf{0.010$\pm$0.000} & 0.127$\pm$0.000        & \textbf{0.019$\pm$0.000} \\
  {\anls}   & 0.229$\pm$0.000    & 0.168$\pm$0.004    & 0.306$\pm$0.000          & \textbf{0.010$\pm$0.000} & 0.127$\pm$0.000        & \textbf{0.019$\pm$0.000} \\
  {\wan}    & 0.411$\pm$0.060    & 0.299$\pm$0.035    & 0.285$\pm$0.016          & 0.156$\pm$0.001          & 0.191$\pm$0.006        & 0.102$\pm$0.000          \\
  {\epr}     & 0.370$\pm$0.043    & 0.220$\pm$0.026    & 0.234$\pm$0.007          & 0.016$\pm$0.000          & 0.136$\pm$0.002        & 0.020$\pm$0.000          \\
  {\role}    & 0.256$\pm$0.018    & 0.176$\pm$0.017    & 0.279$\pm$0.010          & \textbf{0.010$\pm$0.000} & 0.128$\pm$0.000        & \textbf{0.019$\pm$0.000} \\
  {\emm}      & 0.770$\pm$0.001    & 0.820$\pm$0.003    & 0.669$\pm$0.025          & 0.589$\pm$0.003          & 0.718$\pm$0.010        & 0.630$\pm$0.005          \\
  {\emapl} & 0.707$\pm$0.088    & 0.780$\pm$0.082    & 0.641$\pm$0.032          & 0.648$\pm$0.006          & 0.754$\pm$0.017        & 0.622$\pm$0.006          \\
  {\smilee}   & 0.219$\pm$0.009    & 0.182$\pm$0.021    & \textbf{0.208$\pm$0.002} & \textbf{0.010$\pm$0.000} & 0.127$\pm$0.001        & 0.081$\pm$0.008          \\
  {\smib}     & 0.179$\pm$0.004       & 0.211$\pm$0.004       & 0.305$\pm$0.000               & \textbf{0.010$\pm$0.000}      & 0.127$\pm$0.000            & \textbf{0.019$\pm$0.000}          \\ \midrule
  {\llr}  & 0.220$\pm$0.013 & 0.162$\pm$0.005 & 0.312$\pm$0.001 & 0.015$\pm$0.001 & 0.124$\pm$0.002 & \textbf{0.019$\pm$0.000} \\
  {\llcp} & 0.218$\pm$0.016 & 0.164$\pm$0.002 & 0.306$\pm$0.000 & 0.016$\pm$0.001 & 0.126$\pm$0.001 & \textbf{0.019$\pm$0.000} \\
  {\llct} & 0.246$\pm$0.031 & 0.176$\pm$0.019 & 0.321$\pm$0.001 & 0.018$\pm$0.001 & 0.124$\pm$0.001 & \textbf{0.019$\pm$0.000} \\ \midrule
  {\proposed} & \textbf{0.165$\pm$0.023} & \textbf{0.140$\pm$0.013} & 0.211$\pm$0.001 & \textbf{0.010$\pm$0.000} & \textbf{0.121$\pm$0.002} & \textbf{0.019$\pm$0.000} \\ \bottomrule
  \end{tabular}%
  }
  \end{table*}

  \begin{table*}[t]
  \caption{Predictive performance of each comparing methods on MLL datasets in terms of \textit{One-error} (mean $ \pm $ std). The best performance is highlighted in bold (the smaller the better).}
  \centering
  \label{one_error}
  \resizebox{0.8\columnwidth}{!}{%
  \begin{tabular}{@{}ccccccc@{}}
  \toprule
           & Image       & Scene       & Yeast       & Corel5k    & Mirflickr   & Delicious    \\ \midrule
   {\an}      & 0.708$\pm$0.096 & 0.626$\pm$0.123 & 0.489$\pm$0.194 & 0.758$\pm$0.002 & 0.358$\pm$0.005 & 0.410$\pm$0.012 \\
  {\anls}    & 0.643$\pm$0.052 & 0.578$\pm$0.111 & 0.495$\pm$0.130 & 0.736$\pm$0.009 & 0.360$\pm$0.015 & 0.454$\pm$0.013 \\
  {\wan}    & 0.670$\pm$0.060 & 0.543$\pm$0.060 & 0.239$\pm$0.002 & 0.727$\pm$0.012 & 0.352$\pm$0.010 & 0.404$\pm$0.002 \\
  {\epr}    & 0.703$\pm$0.046 & 0.615$\pm$0.090 & 0.240$\pm$0.003 & 0.764$\pm$0.000 & 0.362$\pm$0.015 & 0.441$\pm$0.008 \\
  {\role}  & 0.605$\pm$0.041 & 0.507$\pm$0.066 & 0.244$\pm$0.005 & 0.705$\pm$0.016 & 0.525$\pm$0.072 & 0.594$\pm$0.006 \\
  {\emm}      & 0.769$\pm$0.036 & 0.681$\pm$0.119 & 0.326$\pm$0.079 & 0.656$\pm$0.009 & 0.365$\pm$0.008 & 0.446$\pm$0.009 \\
  {\emapl}  & 0.773$\pm$0.045 & 0.812$\pm$0.059 & 0.341$\pm$0.109 & 0.690$\pm$0.007 & 0.434$\pm$0.023 & 0.405$\pm$0.006 \\
  {\smilee} & 0.533$\pm$0.036          & 0.466$\pm$0.117          & 0.250$\pm$0.012          & 0.650$\pm$0.008   & 0.340$\pm$0.010          & \textbf{0.402$\pm$0.005} \\
  {\smib}   & 0.551$\pm$0.011     & 0.471$\pm$0.010     & 0.276$\pm$0.015     & 0.708$\pm$0.020       & 0.341$\pm$0.012               & 0.529$\pm$0.012            \\ \midrule
  {\llr}  & 0.597$\pm$0.084 & 0.490$\pm$0.054 & 0.436$\pm$0.087 & 0.715$\pm$0.006 & 0.342$\pm$0.016 & 0.543$\pm$0.041 \\
  {\llcp} & 0.629$\pm$0.043 & 0.450$\pm$0.051 & 0.240$\pm$0.000 & 0.731$\pm$0.016 & 0.357$\pm$0.016 & 0.490$\pm$0.028 \\
  {\llct} & 0.616$\pm$0.019 & 0.574$\pm$0.074 & 0.552$\pm$0.097 & 0.726$\pm$0.022 & 0.375$\pm$0.012 & 0.475$\pm$0.019 \\ \midrule
  {\proposed}   & \textbf{0.325$\pm$0.026} & \textbf{0.311$\pm$0.047} & \textbf{0.227$\pm$0.004} & \textbf{0.646$\pm$0.006}          & \textbf{0.295$\pm$0.009} & \textbf{0.402$\pm$0.003}          \\ \bottomrule
  \end{tabular}%
  }
  \end{table*}

\begin{figure*}[t]
    \centering
    \begin{subfigure}{0.24\linewidth}
    \centering
    \includegraphics[width=\linewidth]{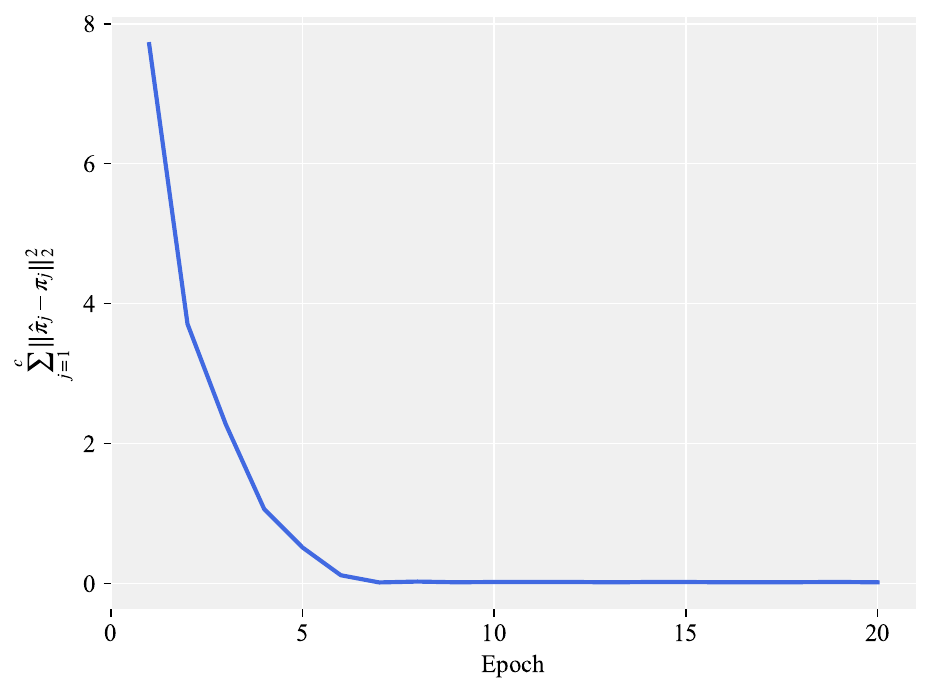}
    \caption{VOC}
    \label{voc}
    \end{subfigure}
    \hfill
    \begin{subfigure}{0.24\linewidth}
    \centering
    \includegraphics[width=\linewidth]{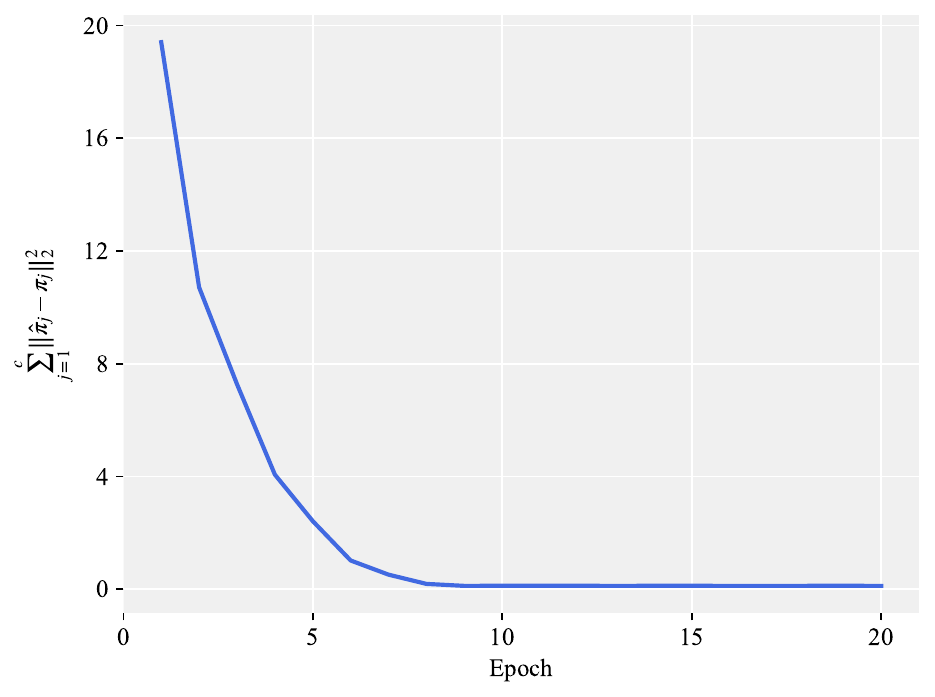}
    \caption{COCO}
    \label{coco}
    \end{subfigure}
    \hfill
    \begin{subfigure}{0.24\linewidth}
        \centering
        \includegraphics[width=\linewidth]{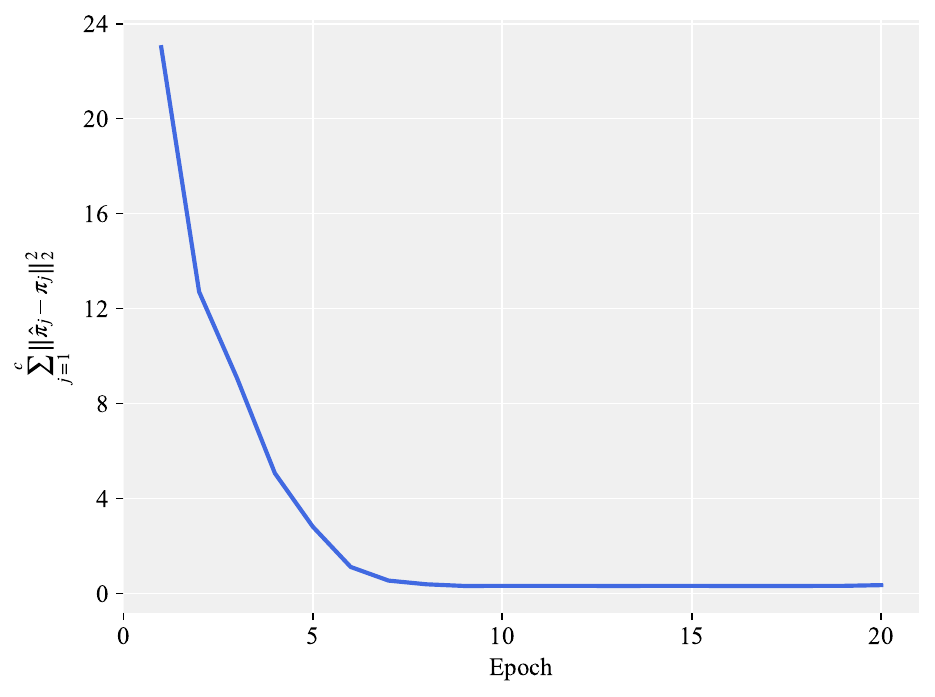}
        \caption{NUS}
        \label{nus}
    \end{subfigure}
    \hfill
    \begin{subfigure}{0.24\linewidth}
        \centering
        \includegraphics[width=\linewidth]{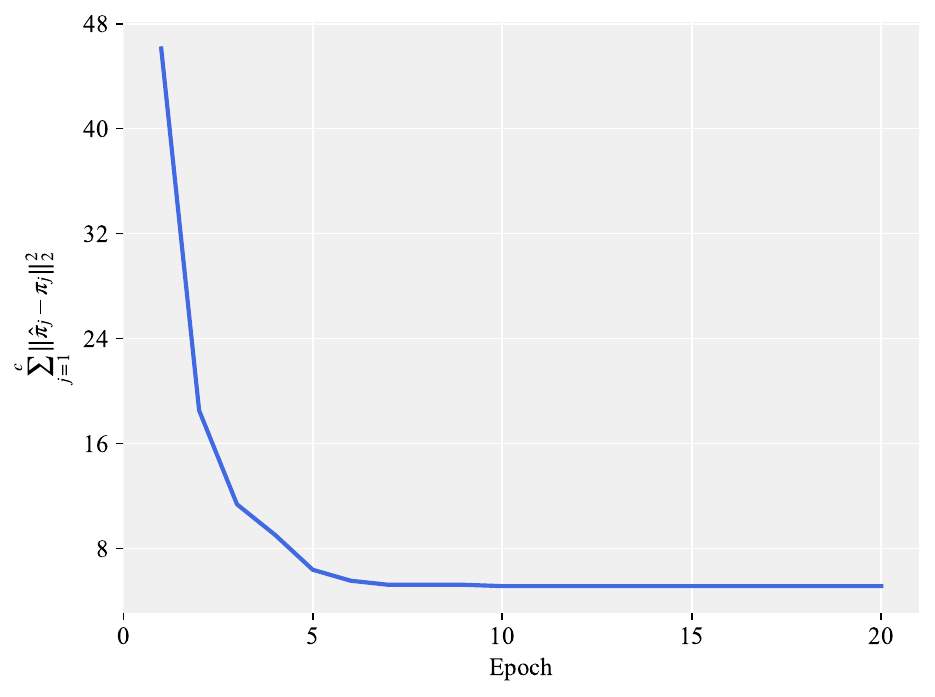}
        \caption{CUB}
        \label{cub}
    \end{subfigure}
    \caption{Convergence of $\hat{\pi}$ on four MLIC datasets.}
    \label{Convergence}
\end{figure*}

  \begin{table*}[!t]
  \centering
  \caption{Summary of the Wilcoxon signed-ranks test for \proposed~against other comparing approaches at 0.05 signiﬁcance level. The \textit{p}-values are shown in the brackets.}
  \label{Wilcoxon}
  \resizebox{\columnwidth}{!}{%
  \begin{tabular}{@{}cccccccccccccc@{}}
  \toprule
  {\proposed} against  & {\an}              & {\anls}           & {\wan}             & {\epr}            & {\role}            & {\emm}              & {\emapl}         & {\smilee} & {\smib} & {\llr} & {\llcp}   &{\llct}    \\ \midrule
  \textit{Coverage} & \textbf{win}{[}0.0313{]} & \textbf{win}{[}0.0313{]} & \textbf{win}{[}0.0313{]} & \textbf{win}{[}0.0313{]} & \textbf{win}{[}0.0313{]} & \textbf{win}{[}0.0313{]} & \textbf{win}{[}0.0313{]} & \textbf{win}{[}0.0313{]} & \textbf{tie}{[}0.0679{]} & \textbf{win}{[}0.0431{]} & \textbf{win}{[}0.0431{]} & \textbf{win}{[}0.0431{]} \\
  \textit{One-error}    & \textbf{win}{[}0.0313{]} & \textbf{win}{[}0.0313{]} & \textbf{win}{[}0.0313{]} & \textbf{win}{[}0.0313{]} & \textbf{win}{[}0.0313{]} &  \textbf{win}{[}0.0313{]} & \textbf{win}{[}0.0313{]} & \textbf{win}{[}0.0431{]} & \textbf{win}{[}0.0313{]} & \textbf{tie}{[}0.625{]} & \textbf{win}{[}0.0313{]} & \textbf{win}{[}0.0313{]} \\
  \textit{Ranking loss} & \textbf{win}{[}0.0313{]} & \textbf{win}{[}0.0313{]} & \textbf{win}{[}0.0313{]}  & \textbf{win}{[}0.0313{]} & \textbf{win}{[}0.0313{]} & \textbf{win}{[}0.0313{]} & \textbf{win}{[}0.0313{]} & \textbf{win}{[}0.0313{]} & \textbf{win}{[}0.0431{]} & \textbf{win}{[}0.0313{]} & \textbf{win}{[}0.0313{]} & \textbf{win}{[}0.0313{]} \\
  \textit{Hamming loss} & \textbf{tie}{[}0.0679{]} & \textbf{tie}{[}0.0679{]} & \textbf{win}{[}0.0313{]}  & \textbf{win}{[}0.0313{]} & \textbf{tie}{[}0.0679{]} & \textbf{win}{[}0.0313{]} & \textbf{win}{[}0.0313{]} & \textbf{tie}{[}0.0796{]} & \textbf{win}{[}0.0313{]} & \textbf{win}{[}0.0313{]} & \textbf{win}{[}0.0313{]} & \textbf{win}{[}0.0313{]} \\
  \textit{Average precision}     & \textbf{win}{[}0.0313{]} & \textbf{win}{[}0.0313{]} & \textbf{win}{[}0.0313{]}  & \textbf{win}{[}0.0431{]} & \textbf{win}{[}0.0313{]} & \textbf{win}{[}0.0313{]} & \textbf{win}{[}0.0313{]} & \textbf{tie}{[}0.0938{]} & \textbf{win}{[}0.0313{]} & \textbf{win}{[}0.0313{]} & \textbf{win}{[}0.0313{]} & \textbf{win}{[}0.0313{]} \\ \bottomrule
  \end{tabular}%
  }
  \end{table*}

  \begin{table*}[!t]
    \centering
    \caption{Predictive performance of {\proposed} compared with 
        the approach of estimating priors from the validation set (\textsc{Crisp-val}) on the MLL datasets for five metrics.}
    \label{val-mll}
    \resizebox{\textwidth}{!}{%
    \begin{tabular}{@{}cccccccc@{}}
    \toprule
    \textbf{} & Metrics & Image & Scene & Yeast & Corel5k & Mirflickr & Delicious \\ \midrule
    \multirow{5}{*}{\proposed} & \textit{Coverage} & \textbf{0.164$\pm$0.012} & \textbf{0.082$\pm$0.018} & \textbf{0.455$\pm$0.002} & \textbf{0.276$\pm$0.002} & \textbf{0.324$\pm$0.001} & \textbf{0.620$\pm$0.001} \\
     & \textit{Ranking Loss} & \textbf{0.164$\pm$0.027} & \textbf{0.112$\pm$0.021} & \textbf{0.164$\pm$0.001} & \textbf{0.113$\pm$0.001} & \textbf{0.118$\pm$0.001} & \textbf{0.122$\pm$0.000} \\
     & \textit{Average Precision} & \textbf{0.749$\pm$0.037} & \textbf{0.795$\pm$0.031} & \textbf{0.758$\pm$0.002} & \textbf{0.304$\pm$0.003} & \textbf{0.628$\pm$0.003} & \textbf{0.319$\pm$0.001} \\
     & \textit{Hamming Loss} & \textbf{0.165$\pm$0.023} & \textbf{0.140$\pm$0.013} & \textbf{0.211$\pm$0.001} & \textbf{0.010$\pm$0.000} & \textbf{0.121$\pm$0.002} & \textbf{0.019$\pm$0.000} \\
     & \textit{OneError} & \textbf{0.325$\pm$0.026} & \textbf{0.311$\pm$0.047} & \textbf{0.227$\pm$0.004} & \textbf{0.646$\pm$0.006} & \textbf{0.295$\pm$0.009} & \textbf{0.402$\pm$0.003} \\ \midrule
    \multirow{5}{*}{\textsc{Crisp-val}} & \textit{Coverage} & 0.193$\pm$0.009 & 0.109$\pm$0.012 & 0.456$\pm$0.004 & 0.280$\pm$0.002 & 0.330$\pm$0.001 & 0.623$\pm$0.002 \\
     & \textit{Ranking Loss} & 0.198$\pm$0.016 & 0.116$\pm$0.013 & 0.165$\pm$0.001 & 0.114$\pm$0.002 & 0.120$\pm$0.001 & \textbf{0.122$\pm$0.000} \\
     & \textit{Average Precision} & 0.725$\pm$0.004 & 0.790$\pm$0.028 & 0.753$\pm$0.006 & 0.294$\pm$0.008 & 0.622$\pm$0.001 & \textbf{0.319$\pm$0.001} \\
     & \textit{Hamming Loss} & 0.180$\pm$0.006 & 0.141$\pm$0.014 & 0.216$\pm$0.000 & \textbf{0.010$\pm$0.000} & 0.124$\pm$0.001 & \textbf{0.019$\pm$0.000} \\
     & \textit{OneError} & 0.395$\pm$0.071 & 0.359$\pm$0.050 & 0.246$\pm$0.021 & 0.666$\pm$0.008 & 0.314$\pm$0.003 & 0.444$\pm$0.001 \\ \bottomrule
    \end{tabular}%
    }
    \end{table*}

  \subsection{Implementation Details}\label{implementation}
  During the implementation, we ﬁrst initialize the predictive network by performing warm-up training with AN solution, which could facilitate learning a ﬁne network in the early stages. Furthermore, after each epoch, the class prior is reestimated via the trained model. The code implementation is based on PyTorch, and the experiments are conducted on GeForce RTX 3090 GPUs. 
  The batch size is selected from $ \{8, 16, 32\} $ and the number of epochs is set to $ 10 $. We use Adam as the optimizer and the learning rate and weight decay are selected from $ \{10^{-2}, 10^{-3}, 10^{-4}, 10^{-5}\} $. The hyperparameters $ \delta $ and $ \tau $ are all fixed as $ 0.01 $. All hyperparameters are selected with a validation dataset. All the comparing methods run 5 trials on each datasets. For fairness, we employed ResNet-50 as the backbone for all comparing methods.
  
  \subsection{Details of Datasets}\label{DS_details}
  The details of the four MLIC datasets and the five MLL datasets are provided in Table \ref{image_ds} and Table \ref{mll_ds} respectively. The basic statics about the MLIC datasets include the number of training set, validation set, and testing set (\#Training, \#Validation, \#Testing), and the number of classes (\#Classes). The basic statics about the MLL datasets include the number of examples (\#Examples), the dimension of features (\#Features), the number of classes (\#Classes), and the domain of the dataset (\#Domain).

  \subsection{More Results of MLL Datasets}\label{MLL_results}
  Table \ref{coverage}, \ref{hamming} and \ref{one_error} report the results of our method and other comparing methods on five MLL datasets in terms of \textit{Coverage, Hamming loss} and \textit{One Error} respectively. 
  
  \subsection{Parameter Analysis for $ \tau $ and $ \lambda $}\label{app:sens}
  Parameter analysis for $ \tau $ and $ \lambda $ are provided in Figures \ref{tau} and \ref{lambda} respectively.
  \begin{figure*}[t!]
    \centering
    \begin{subfigure}{0.4\linewidth}
    \centering
    \includegraphics[width=\linewidth]{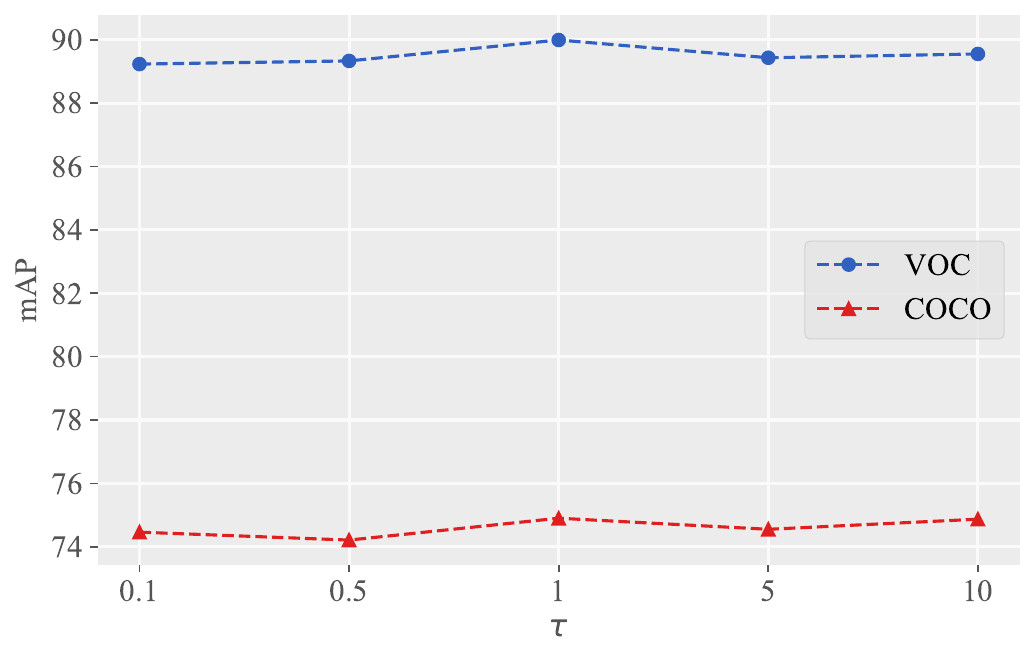}
    \caption{Sensitivity analysis of $\tau$}
    \label{tau}
    \end{subfigure}
    \begin{subfigure}{0.4\linewidth}
    \centering
    \includegraphics[width=\linewidth]{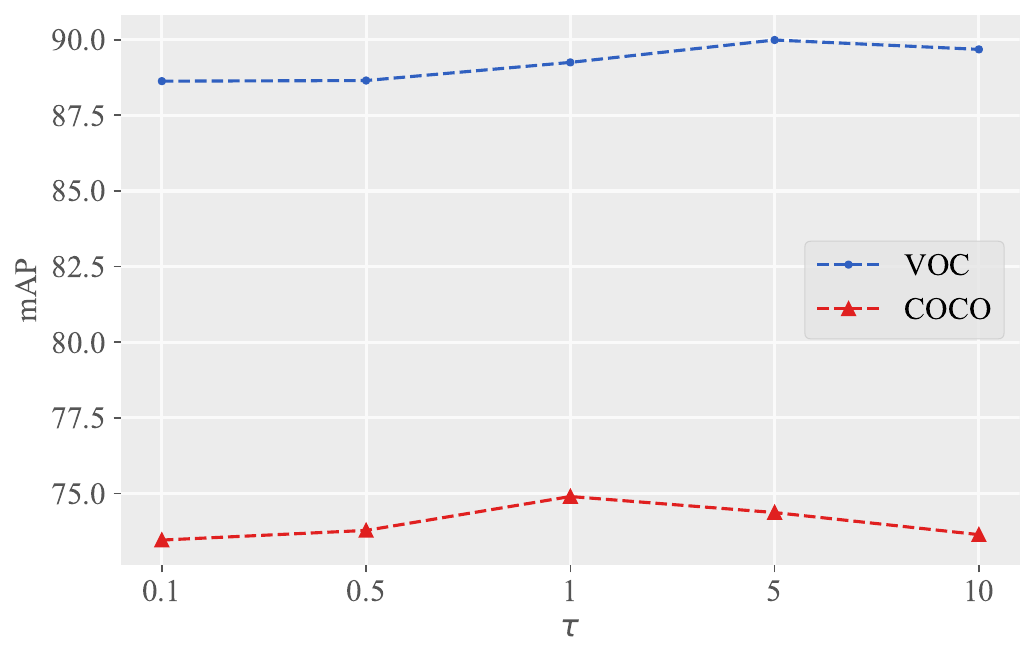}
    \caption{Sensitivity analysis of $\lambda$}
    \label{lambda}
    \end{subfigure}
    \caption{
    (a) Parameter sensitivity analysis of $\tau$ ($ \delta $ is fixed as $ 0.01 $, $ \lambda $ is fixed as $ 1 $); 
    (b) Parameter sensitivity analysis of $\lambda$ ($ \tau $ and $ \delta $  are fixed as $ 0.01 $).}
\end{figure*}

  \subsection{More Results of MLIC Datasets}\label{MLIC_results}
  
  Figure \ref{Convergence} illustrates the discrepancy between the estimated class-prior $\hat{\pi}_j$ and the true class-prior $\pi_j$ in every epoch on four MLIC datasets. During the initial few epochs, a significant decrease in the discrepancy between the estimated class-prior and the true class-prior is observed. After several epochs, the estimated class prior tends to stabilize and converges to the true class-prior. This result provides evidence that our proposed method effectively estimates the class-prior with the only observed single positive label. 

  \subsection{\textit{p}-values of the wilcoxon signed-ranks test} \label{app:wilcoxon}
  Table \ref{Wilcoxon} reports the \textit{p}-values of the wilcoxon signed-ranks test \citep{demvsar2006statistical} for the corresponding tests and the statistical test results at 0.05 signiﬁcance level.

  \subsection{Ablation results of MLL datasets}\label{app:val-mll}
  Table \ref{val-mll} reports the predictive performance of {\proposed} compared with the approach of estimating priors from the validation set (\textsc{Crisp-val}) on the MLL datasets for five metrics. The results show that {\proposed} outperforms \textsc{Crisp-val} on almost all the five metrics.

\section*{Impact Statements}
This research aims to advance the techniques and methods in the field of Machine Learning. Our approach could potentially result in the displacement of data annotators or other individuals involved in data-related occupations. We recognize the importance of addressing the implications of automation on employment and are mindful of its societal impacts.

\end{document}